\documentclass{article}

\usepackage{microtype}
\usepackage{graphicx}
\usepackage{subfigure}
\usepackage{booktabs} %
\usepackage{adjustbox}

\usepackage{hyperref}

\usepackage[accepted]{icml2023}

\usepackage{amsmath}
\usepackage{amssymb}
\usepackage{mathtools}
\usepackage{amsthm}
\usepackage{longtable}%
\usepackage{tabularx}
\usepackage{tikz}
\usepackage{tikz-cd}
\usetikzlibrary{automata, arrows}
\usetikzlibrary{positioning}
\usepackage{pbox}
\usepackage{verbatim}
\usepackage{mathtools}
\usepackage[T1]{fontenc}
\usepackage[utf8]{inputenc}
\usepackage{amssymb}
\usepackage[normalem]{ulem}
\usepackage{soul}
\usepackage{bbm}
\usepackage{framed}
\usepackage{color}
\usepackage{amsmath}
\usepackage{amsthm}
\usepackage{thmtools, thm-restate}
\usepackage{amsmath}
\usepackage{amssymb}
\usepackage{cancel}
\usepackage{changepage}
\usepackage{tikz}
\usepackage{bbm}
\usetikzlibrary{automata, arrows}
\usetikzlibrary{positioning}
\usepackage{mathtools}
\usepackage{mathrsfs}
\usepackage{enumerate}
\usepackage{multicol}
\usepackage{upgreek}
\usepackage{bigints}
\usepackage{physics}
\usepackage{accents}
\usepackage[toc,page,title]{appendix}
\usepackage{bm}
\usepackage{enumitem}
\usepackage{booktabs}
\usepackage{graphics}
\usepackage{blindtext}
\usepackage{physics}
\usepackage{todonotes}
\usepackage{hyperref}
\hypersetup{
    colorlinks=true,
    linkcolor=blue,
    filecolor=magenta,      
    urlcolor=cyan,
    pdftitle={Overleaf Example},
    pdfpagemode=FullScreen,
    }

\newcommand{\cD}{\mathcal{D}}

\newcommand{\cF}{\mathcal{F}}

\newcommand{\cM}{\mathcal{M}}

\newcommand{\cP}{\mathcal{P}}

\newcommand{\cR}{\mathcal{R}}
\newcommand{\cS}{\mathcal{S}}

\newcommand{\cV}{\mathcal{V}}

\newcommand{\cX}{\mathcal{X}}

\newcommand{\bE}{\mathbb{E}}

\newcommand{\bM}{\mathbb{M}}
\newcommand{\bN}{\mathbb{N}}

\newcommand{\bP}{\mathbb{P}}
\newcommand{\bQ}{\mathbb{Q}}
\newcommand{\bR}{\mathbb{R}}

\newcommand{\bd}{\mathbbm{d}}

\newcommand{\bn}{\mathbbm{n}}

\renewcommand{\phi}{\varphi}

\usepackage{empheq}
\usepackage[most]{tcolorbox}
\newtcbox{\mymath}[1][]{%
    nobeforeafter, math upper, tcbox raise base,
    enhanced, colframe=blue!30!black,
    colback=blue!30, boxrule=1pt,
    #1}

\DeclareMathOperator{\supp}{supp}
\DeclareMathOperator{\argmin}{argmin}

\DeclareMathOperator{\Dists}{Dists}

\DeclareMathOperator{\sign}{sign}

\definecolor{emerald}{rgb}{0.31, 0.78, 0.47}
\usepackage[capitalize,noabbrev]{cleveref}

\theoremstyle{plain}
\newtheorem{theorem}{Theorem}[section]

\newtheorem{lemma}[theorem]{Lemma}
\newtheorem{corollary}[theorem]{Corollary}
\theoremstyle{definition}
\newtheorem{definition}[theorem]{Definition}
\newtheorem{assumption}[theorem]{Assumption}
\theoremstyle{remark}

\newcommand{\LS}{\theta_{\texttt{LS}}}
\newcommand{\LSTD}{\theta_{\texttt{LSTD}}}
\newcommand{\vLS}{v_{\texttt{LS}}}
\newcommand{\vLSTD}{v_{\texttt{LSTD}}}

\newcommand{\Sigmacr}{\Sigma_{\texttt{cr}}}
\newcommand{\sigmamin}[1]{\sigma_{\min}(#1)}
\DeclareMathOperator{\col}{col}

\DeclareMathOperator{\Ker}{ker}

\newcount\Comments  %
\definecolor{darkred}{rgb}{0.7,0,0}
\definecolor{darkgreen}{rgb}{0,0.5,0}
\definecolor{orange}{rgb}{0.7,0.4,0}
\definecolor{purple}{rgb}{0.8,0.0,0.8}

\icmltitlerunning{The Optimal Approximation Factors in Misspecified Off-Policy Value Function Estimation}

\begin{document}
\twocolumn[
\icmltitle{The Optimal Approximation Factors in \\ Misspecified Off-Policy Value Function Estimation}

\icmlsetsymbol{equal}{*}

\begin{icmlauthorlist}
\icmlauthor{Philip Amortila}{yyy}
\icmlauthor{Nan Jiang}{yyy}
\icmlauthor{Csaba Szepesv\'ari}{sch}
\end{icmlauthorlist}

\icmlaffiliation{yyy}{University of Illinois, Urbana-Champaign}
\icmlaffiliation{sch}{University of Alberta}

\icmlcorrespondingauthor{Philip Amortila}{philipa4@illinois.edu}

\icmlkeywords{Machine Learning, ICML}

\vskip 0.3in
]

\printAffiliationsAndNotice{}  %

\begin{abstract}
Theoretical guarantees in reinforcement learning (RL) are known to suffer multiplicative blow-up factors with respect to the misspecification error of function approximation. Yet, the nature of such \emph{approximation factors}---especially their optimal form in a given learning problem---is poorly understood. In this paper we study this question in linear off-policy value function estimation, where many open questions remain. We study the approximation factor in a broad spectrum of settings, such as with the weighted $L_2$-norm (where the weighting is the offline state distribution), the $L_\infty$ norm, the presence vs. absence of state aliasing, and full vs. partial coverage of the state space.  We establish the optimal asymptotic approximation factors (up to constants) for all of these settings. In particular, our bounds identify two instance-dependent factors for the $L_2(\mu)$ norm and only one for the $L_\infty$ norm, which are shown to dictate the hardness of off-policy evaluation under misspecification.  %
\end{abstract}

\section{Introduction}
\label{sec:intro}

Realizability assumptions are pervasive amongst theoretical guarantees in reinforcement learning (RL) with function approximation. These assumptions posit that the true optimal solution, a value function to be estimated from data, belongs to the function class which is used. %
In practice, however, the realizability assumption rarely holds, and the degree to which it is violated %
	is largely unknown. Thus, we need
	algorithms that do not rely on the realizability assumption in the sense that their guarantees \emph{automatically} scale with the degree of misspecification. 

When the ground truth solution is not representable by the function class, a natural relaxed objective is to instead recover the \emph{best-in-class function} in the function class, i.e. the function which is closest to the true solution as measured by some norm. The ``minimal'' error incurred by the best-in-class function is called the \emph{misspecification} error. The ratio between the error of the attained solution and that of the best-in-class solution is called the \textit{approximation factor} or approximation ratio.

Existing error bounds for misspecified RL problems often suffer large approximation factors in addition to other statistical errors \cite{chen2019information,xie2021batch}. Unlike the statistical errors, these error terms represent the ``bias'' of the solution, and thus do not decrease even asymptotically as the sample size goes to infinity. %
It is rarely the case that attention is brought to whether these blowup factors are necessary, or if the factors attained are optimal. 

In a myriad of settings which are simpler than RL (such as in linear regression or empirical risk minimization), it is indeed possible to recover an approximation factor of $1$ (or arbitrarily close to $1$) \cite{wainwright2019high,shalev2014understanding}. Whether or not similar guarantees are possible in RL problems, or what the optimal ratios would be, has been largely unstudied. Towards studying this question, we formulate an offline RL problem with linear features, and examine the optimal approximation ratio achieved by any estimator (even \textit{asymptotic} ones). Attainable approximation factors may depend on the number of samples available, but the optimal asymptotic approximation factor is as low as it can be since even ``sample-inefficient'' estimators are allowed.

Concretely, our learning problem is that of linear off-policy value function estimation in infinite-horizon discounted Markov Reward Processes (MRPs). Despite the apparent simplicity of this setting, even here an understanding of the blowup factors remains open. In this problem, the learner is given access to a feature-map $\phi: \cS \rightarrow \bR^d$ and an offline dataset of tuples $(s_i,r_i,s'_i)$ from the MRP. The states $s_i$ are sampled i.i.d. from an off-policy distribution $\mu$ that may be different from the stationary distribution of the MRP. We also study the \textit{aliased} setting where the states can only be observed through their feature mapping. We do not assume anything about the off-policy distribution beyond that it yields a non-degenerate second moment matrix (defined in Section \ref{sec:background}).\footnote{e.g. $\mu$ need not cover the entire state space or have good ``concentrability'' with respect to the stationary distribution}  We also do not assume that the value function of the MRP is linear in the given feature mapping, and thus the task of the learner is simply to output the \emph{best possible linear approximation} of the true value function (as measured by some norm). Our question is thus: \emph{``what is the optimal asymptotic approximation factor for linear off-policy value function estimation under misspecification?''}

Recent works \cite{amortila2020variant,perdomo_sharp_2022} have provided some negative results in the realizable setting which demonstrate that the approximation factor may be arbitrarily large \textit{in the worst case}.
In this paper, we provide instance-dependent upper and lower bound results, with the goal of pinning down the optimal approximation ratio for off-policy value function estimation under both the $L_2(\mu)$ norm and the $L_\infty$ norm. For upper bounds, we analyze the well-known (off-policy) Least Squares Temporal Difference (LSTD) algorithm \cite{bradtke1996linear}, and provide exact characterizations of its error compared to the optimal linear projection. %
This leads to an approximation factor for LSTD involving two problem-dependent terms, giving two ``failure modes'' for this algorithm. Via instance-dependent lower bounds, we show that the approximation factor attained by LSTD is optimal (up to constant factors) in a myriad of settings of interest. In other settings, we instead show that alternative model-based estimators attain the optimal approximation factors.  For ease of reference, a summary of the settings that we study and their associated results can be found in Table \ref{fig:table-results}.

Our results explain the previous unidentifiability results, as well as provides new ones. %
To our knowledge, the only prior work establishing the optimality of LSTD was in the \textit{on-policy} setting (i.e., when $\mu$ is the stationary distribution) and held for sample sizes which were much smaller than the size of the state space \cite{mou_optimal_2020}. In particular, no prior work exists on characterizing the necessary blowup of the misspecification error in the off-policy case, even that which is asymptotically achievable. Furthermore, while prior LSTD bounds are primarily in $L_2(\mu)$ norm, we also provide additional results in the $L_\infty$ norm, which allows for distribution-free error guarantees that are crucial for RL. 
\paragraph{Differences from ICML version} Compared to the ICML camera-ready, the present (arXiv) version of this paper contains improved results. In particular, we have added Subsection \ref{sec:l2munoalias} and Theorem \ref{thm:l2mulocal}, which completes all of the entries of Table \ref{fig:table-results}. The ICML version only established limiting cases for that particular setting. As this new result resolves questions left open by the ICML version, the ensuing discussions are also modified. The text is otherwise identical to the ICML version (modulo minor improvements).

\begin{table*}[t]
\centering
\begin{adjustbox}{width=0.7\textwidth}
\begin{tabular}{@{}l|c|c@{}}
\toprule
& $L_2(\mu)$ norm & $L_\infty$ norm  \\ 
\midrule
$\mu \geq 0$. Aliasing. & $\alpha^\star \approx \sqrt{ 1 + \left(\gamma \frac{\norm{\Pi_\mu P}_\mu} {\sigmamin{\Sigma^{-1/2}A\Sigma^{-1/2}}}\right)^2}$  & $\alpha^\star \approx 1 + \frac{1+\gamma}{\sigmamin{A}}$\\
\midrule 
$\mu \geq 0$. No aliasing. & $\alpha^\star \approx \sqrt{ 1 + \left(\gamma \frac{\norm{\Pi_\mu P}_\mu} {\sigmamin{\Sigma^{-1/2}A\Sigma^{-1/2}}}\right)^2}$ & $\alpha^\star \approx 1 + \frac{1+\gamma}{\sigmamin{A}}$ \\
\midrule
$\mu > 0$. Aliasing. & $\alpha^\star \approx \sqrt{ 1 + \left(\gamma \frac{\norm{\Pi_\mu P}_\mu} {\sigmamin{\Sigma^{-1/2}A\Sigma^{-1/2}}}\right)^2}$ & $\alpha^\star = \frac{1}{2(1-\gamma)}$ \\ %
\midrule 
$\mu > 0$. No aliasing. & $\alpha^\star = 1$. & $\alpha^\star = 1$.  \\
\bottomrule
\end{tabular}
\end{adjustbox}
\caption{The optimal asymptotic approximation factors $\alpha^\star$ for various settings. $\mu \geq 0$: offline distribution is arbitrary. $\mu > 0$: offline distribution has full support. Aliasing: states are only observed through feature mapping (cf. Section \ref{sec:setup}). The terms $\Pi_\mu$, $\Sigma$, and $A$ are defined in Section \ref{sec:background}.
$\approx$ indicates matching upper and lower bounds up to constants for certain parameter regimes. $=$ indicates matching upper and lower bounds.
}
\label{fig:table-results}
\end{table*}

\section{Problem setup}\label{sec:setup}

This section formalizes   linear off-policy value function estimation in discounted Markov Reward Processes.

\paragraph{Notation} We write $\Dists(\cX)$ to denote the set of probability distributions over a set $\cX$. We write $I_{n \times n}$ for the $n \times n$ identity matrix, or simply $I$ when the dimension is clear from context. For any matrix $X$, we let $\lambda_{\min}(X)$ and $\sigma_{\min}(X)$ denote its minimum eigenvalue (if $X$ is square) and minimum singular value, respectively. All vectors are column vectors, and we write $^\top$ for the transpose operator.

\paragraph{Markov Reward Processes} Markov Reward Processes arise when a fixed memoryless policy is followed in a
a Markov Decision Process \cite{puterman2014markov,szepesvari2010algorithms}. %

\begin{definition}[Markov Reward Process]
A finite discounted Markov Reward Process (MRP) $\cM = \langle \cS, \cR, \cP, \gamma \rangle$ is defined by a finite state space $\cS \in \bN$, a stochastic reward function $\cR: \cS \rightarrow \Dists([-1,1])$ with expectation $r(s) = \int x \dd(\cR(s))$, a transition function $\cP: \cS \rightarrow \Dists(\cS)$, and a discount factor $\gamma \in [0,1)$. %
\end{definition}
To simplify the presentation, %
we consider finite (but arbitrarily large) state spaces. As is standard, we have assumed that the reward distribution at any state is almost-surely bounded. We will write $S \coloneqq |\cS|$, and canonically identify $\cS = \{1,\cdots,S\}$. We can identify $r$ with a ${S}$-dimensional vector and $\cP$ with the $S\times S$ row-stochastic matrix. We write $\cP(s'\vert s) = [\cP(s)](s') = P_{s,s'}$. %
 The value function of an MRP is the following:

\begin{definition}[Value function]
The value function in an MRP $\cM$ is the function $v_\cM: \cS \mapsto [\frac{-1}{1-\gamma}, \frac{1}{1-\gamma}]$ defined by:
\[
v_\cM(s) = \bE\left[\sum_{t\geq 0} \gamma^t r(S_t) \,\bigg|\, S_0 := s, S_t \sim P(S_{t-1})\right]\,.
\]
In vector notation we have
\[
v_\cM = \sum_{t=0}^\infty \gamma^t P^t r = (I-\gamma P)^{-1} r,
\]
which is an $S$-dimensional vector.
\end{definition}

\paragraph{Policy evaluation with misspecified linear features}

 A feature map $\varphi: \cS \rightarrow \bR^d$ is given, which the learner can use to approximate $v_\cM$.
 The task of the learner is to output a function $f:\cS \to \bR$ that is linear in the features in the sense that  for some $\theta\in \bR^d$, for every $s\in \cS$, $f(s) = \theta^\top \varphi(s)$. We write $\Phi \in \bR^{S \times d}$ for the matrix whose $s^\text{th}$ row $(s \in S)$ is  $(\phi(s))^\top$, and $\cF_\Phi = \{f_\theta = \Phi \theta \mid \theta \in \bR^d\} \subseteq \bR^{S}$ for the subspace consisting of linear functions. The learner will be evaluated by how far the function $f$ is from $v_\cM$ in a given norm, which is also available to the learner. 
 We consider the so-called \emph{misspecified} setting, that is, we do not assume $v_{\cM}$ itself is a linear function of the features. Instead, the learner is only asked to
 produce a function whose error is not much larger than that of the \emph{best linear approximation} of $v_\cM$ in the given norm (obtained via the projection operators defined in \cref{sec:background}). %

\paragraph{Observation model} 

We study the \emph{offline} setting, meaning that the learner is given a dataset $\cD_n$ from the MRP and no interaction is allowed. We will study both the \emph{aliased} and \emph{non-aliased} settings. %
In the aliased setting \cite{sutton2018reinforcement}, the states are only seen through the feature mapping. Formally, the observations take the form of $n$ i.i.d. samples, which are generated by the following process 
\begin{align}
    &\phi_i = \phi(s_i) \text{ where } s_i \stackrel{\texttt{i.i.d.}}{\sim} \mu, \label{eq:mu-def}\\
    &R_i \sim \cR(s_i), \\
    &\phi'_i = \phi(s'_i) \text{ where } s'_i \sim \cP(s_i).
\end{align}
We refer to the joint distribution over the triplets $(\phi_i,r_i,\phi'_i)$ as $\bQ_{\cM,\mu,\phi}$, and thus the dataset $\cD_n = \{ (\phi_i,r_i,\phi'_i) \}_{i=1}^n$ consists of $n$ i.i.d.~samples from $\bQ_{\cM,\mu,\phi}$.

In the \emph{non-aliased} setting, the learner instead observes
$\cD^\diamond_n = \{ (s_i,\phi(s_i),r_i,s_i',\phi(s'_i)) \}_{i=1}^n$, where
\begin{equation}
    s_i \stackrel{\texttt{i.i.d.}}{\sim} \mu, \, r_i \sim \cR(s_i), \, s'_i \sim \cP(S_i). %
\end{equation}
We will refer to the joint distribution over non-aliased tuples $(s_i,\phi(s_i), r_i, s'_i,\phi(s'_i))$ as $\bQ^\diamond_{\cM,\mu,\phi}$. We write $\supp(\mu) \coloneqq \{\mu(s) > 0\} \subseteq \cS$ for the support of $\mu$.

We are in the \emph{off-policy} setting, by which we mean that $\mu$ is \emph{not} restricted to be a stationary distribution of the transition matrix $P$. In particular, we do not assume that $\mu$ has good ``concentrability'' or has support over the entire state space.

All of our upper bounds will apply to the aliased setting and thus also for the easier non-aliased setting, so we will only need to distinguish the settings when stating lower bounds. 
 We make some minor ``quality of life'' assumptions about $\phi$ and $\mu$, which are mainly for convenience. Let us write $D$ for the diagonal matrix with the entries of $\mu$ along its diagonal (i.e. $D_{s,s} = \mu(s)$, for $s \in S$, and $0$ otherwise).

\begin{assumption}[Feature boundedness \& non-degenerate second moment]\label{ass:lrassumptions}
We have $
\max_s \norm{\phi(s)}_2 \leq 1. %
$
Furthermore, we assume that 
$
\Sigma \coloneqq \Phi^\top D \Phi =\bE_\mu[\phi(s)\phi(s)^\top]
$
is invertible. 
\end{assumption}
Above, the $L_2$-boundedness of $\phi$ just provides a normalization of the features and can be assumed without loss of generality. Furthermore, if $\Sigma$ is not invertible then the features are redundant; the dimensionality of the feature space can be reduced so that after the reduction $\Sigma$ is invertible. Hence, this assumption can also be made without loss of generality, and we further know that it is insufficient by itself for the value prediction problem (even under realizability) \cite{amortila2020variant}. %

\paragraph{Optimal asymptotic approximation factors}

In the misspecified setting, the quality of a finite-sample estimator is characterized by its \emph{approximation ratio} and its \emph{statistical error}. 
If, given the dataset $\cD_n$ a learner returns the (possibly random) function $\hat v = \hat v(\cD_n) \in \cF_\Phi$, one often upper bounds the error of the returned function via an \emph{oracle inequality} of the following form: %
\begin{align}
\norm{\hat v - v_\cM} &\leq \underbrace{\alpha_n(\cM,\mu,\phi)}_{\text{approximation factor}} 
\underbrace{\inf_{\theta} \norm{\Phi \theta - v_M}}_{\text{oracle's error}} \nonumber \\
&\,\, + \underbrace{\varepsilon_n(\cM,\mu,\phi)}_{\text{statistical error}} \label{eq:oracle-ineq-def},
 \end{align}
 which holds either with high probability or in expectation. The approximation factor measures the magnification of the oracle
 approximation error $\inf_{\theta}\norm{\Phi \theta - v_\cM}$, and may be due to the imperfection
 of the learning algorithm or because of a fundamental hurdle that every learner faces (or both).
As we will be interested in the fundamental difficulty all learners face in off-policy estimation, regardless of sample-sizes,
we will consider the limit of infinite sample sizes, where the statistical error is zero.
In particular, we can think of this as the case when 
the learner is given the distribution $\bQ_{\cM,\mu,\phi}$ (in the non-aliased case the distribution 
 $\bQ^\diamond_{\cM,\mu,\phi}$). In the non-aliased case this is equivalent to the learner being given the model $\cP(s)$ and $r(s)$ \emph{for all states $s \in \supp(\mu)$}, and its task can be viewed as ``completing'' this model outside of the data distribution (using the features). A learner is a map from distributions of the above form to linear functions over $\cF_\Phi$. The approximation ratio exhibited by a deterministic asymptotic estimator is:%
\begin{equation}\label{eq:approx-ratio-def}
\alpha^{\hat{v}}_{\norm{\cdot}}(\cM,\mu,\phi) = \frac{ \norm{\hat{v}(\bQ_{\cM,\mu,\phi}) - v_\cM} }{ \inf_{\theta} \norm{\Phi \theta - v_\cM} },
\end{equation}
with the convention that $\frac{0}{0}=1$ and $\frac{x}{0}=\infty$ whenever $x > 0$. %
We refer to $\inf_{\theta} \norm{\Phi \theta - v_M}$ as the \emph{misspecification error} of the MRP $\cM$. %
We do not need to consider random asymptotic estimators since, if one measures them by their expected approximation ratio, Jensen's inequality tells us that deterministic estimators are optimal.\footnote{Since the averaged estimator $\bE[\hat{v}]$ will be deterministic and output functions in $\cF_\Phi$, and we have $\norm{ \bE[\hat{v}(\bQ_{\cM,\mu,\phi})] - v_\cM} \leq \bE[\norm{\hat{v}(\bQ_{\cM,\mu,\phi}) - v_\cM}]$.}

We will consider two natural choices for the norms, the weighted $L_2(\mu)$ norm and the $L_\infty$ norm. These are defined by
$$
\norm{v}_\mu = \left(\sum_s \mu(s) v^2(s)\right)^{1/2} \, \quad \& \quad \, \norm{v}_\infty = \max_{s} \lvert v(s) \rvert, 
$$
where $\mu$ is the offline state distribution from Equation \eqref{eq:mu-def}. For any matrix $X \in \bR^{S \times S}$, we will write $\norm{X}_\mu$ for it's $L_2(\mu)$-operator norm. The $L_2(\mu)$ norm is a natural choice for function estimation as it only asks to minimize the error on states which have been encountered. In particular, for the simpler problem of linear regression (a special case of our setting for $\gamma=0$), the least squares estimator attains the optimal approximation ratio of $1$ under this norm. Meanwhile, the $L_\infty$ norm is important for obtaining distribution-independent guarantees which we often need for RL, e.g. when value prediction is being used as a subroutine \cite{lagoudakis2003least}. We emphasize that our problem setting requires \textit{function estimation} (estimating $v_\cM$) rather than simply \textit{return estimation} (estimating $v_\cM$ under an initial distribution). Function estimation is a strictly more difficult problem, and there are many applications where one would require a guarantee on the error of off-policy evaluation on the whole space rather than simply at the initial states, e.g. for the aforementioned subroutines or in model selection problems \cite{huangbeyond}. We will write $\alpha_\mu$ for approximation ratios in the $L_2(\mu)$ norm, and $\alpha_\infty$ for approximation ratios in the $L_\infty$ norm. %

\section{Background}\label{sec:background}

The optimal linear approximations of $v_\cM$ are obtained by taking its projection via the projection operators.

\begin{definition}[Projection operators]
We write $\Pi_\mu$ for the linear projection in the $L_2(\mu)$ norm, i.e. 
$
\Pi_\mu v = \argmin_{\hat{v} \in \cF_\Phi} \norm{\hat{v}- v}_\mu.
$
This operator has a closed form,
\begin{equation}
\Pi_\mu = \Phi \Sigma^{-1} \Phi^\top D,
\end{equation}
which is well-defined by Assumption \ref{ass:lrassumptions}. We also write $\Pi_\infty$ for the linear projection in the $L_\infty$ norm, i.e. 
$
\Pi_\infty v \in \argmin_{\hat{v} \in \cF_\Phi} \norm{\hat{v} - v}_\infty.
$
The $L_\infty$ projection may not be unique, and we consider that ties can be broken arbitrarily (we will not need to refer to a specific minimizer, only the value of the minimum). %
\end{definition}

One canonical estimator for the policy evaluation problem is the Least Squares Temporal Difference (LSTD) algorithm \cite{bradtke1996linear}. 
In the limit of infinite samples, or
\emph{at the population level}, it is defined by the estimator 
\begin{align}
A &\coloneqq \Phi^\top D (I - \gamma P) \Phi = \bE_{s,s'} \left[\phi(s) ( \phi(s) - \gamma \phi(s'))^\top \right] \label{eq:A-matrix-def} \\
b &\coloneqq \Phi^\top D r = \bE_{s \sim \mu}\left[\phi(s) r(s) \right] \label{eq:b-vector-lstd} \\
\LSTD &\coloneqq A^{-1} b\,, \qquad \vLSTD = \Phi \LSTD\,,
 \label{eq:lstd}
\end{align}
whenever $A$ is invertible. In the sequel we will see that we do not need to define $\LSTD$ when $A$ is not invertible since in that case no estimator can have a finite approximation ratio. The finite-sample version of LSTD is obtained by replacing $A$ and $b$ by their empirical averages.  We note that LSTD is applicable in the aliased setting.%

\section{Approximation Ratios in the $L_2(\mu)$ Norm}\label{sec:l2mu}

We begin by studying the optimal approximation factor in the $L_2(\mu)$ norm. Section \ref{sec:l2mu-aliasing} provides a general upper bound for the approximation ratio attained by LSTD and then provides a nearly-matching lower bound for the aliased setting. Section \ref{sec:l2mu-no-aliasing} establishes that this approximation ratio is also optimal in the non-aliased setting. The results of this section are summarized in the ``$L_2(\mu$) norm'' column of Table \ref{fig:table-results}.

\subsection{Under aliasing: LSTD attains the optimal approximation factor}\label{sec:l2mu-aliasing}

Our first result is a tight upper bound for the approximation factor obtained by LSTD. 

\begin{restatable}{theorem}{lstdub}\label{thm:lstd-ub}
Assume that the $A$ matrix from Equation \eqref{eq:A-matrix-def} is invertible. Then the population LSTD estimator of Equation \eqref{eq:lstd} has an approximation factor upper bound of 
\begin{align}
\alpha_\mu^{\text{LSTD}} &\leq \sqrt{1 + \left(\gamma \norm{\Phi A^{-1} \Phi^\top D P}_\mu\right)^2} \label{eq:lstd-ub-sharper}\\
&\leq \sqrt{1 + \left(\gamma \frac{\norm{\Pi_\mu P}_\mu}{\sigmamin{\Sigma^{-1/2}A\Sigma^{-1/2}}}\right)^2} \label{eq:lstd-ub}
\end{align}
\end{restatable}

\begin{proof}[Proof (sketch)]
This result 
relies on an \emph{exact error decomposition} of the LSTD solution: 
\begin{equation} \label{eq:ls-minus-lstd}
\Phi\LS - \Phi\LSTD = \gamma \Phi A^{-1}\Phi^\top D P (\Pi_\mu v_\cM - v_\cM),
\end{equation}
where $\LS$ is the least-squares parameter corresponding to the optimal projection, i.e. satisfying $\Phi \LS = \Pi_\mu v_\cM$. See Appendix \ref{app:lstd-ub} for a full proof. 
\end{proof}

We mention that every occurrence of $\gamma P$ in the above bound can also be replaced by $(I-\gamma P)$, so we can in effect take the minimum over the two bounds. The savings from this are limited since $\left|\norm{ \Pi_\mu (I - \gamma P) }_\mu - \gamma \norm{\Pi_\mu P}_\mu \right|\leq 1$. 
We note that the vector $\Pi_\mu v_\cM - v_\cM$ is the component of the value function which is orthogonal to the features, so Equation \eqref{eq:ls-minus-lstd} indicates that the error of LSTD is precisely dictated by action of the linear operator $\Phi A^{-1} \Phi^\top D P$ on this vector. %
The second upper bound in Theorem \ref{thm:lstd-ub} (Equation \eqref{eq:lstd-ub}) further separates out the two terms $\norm{\Pi_\mu P}_\mu$ and $\sigmamin{\Sigma^{-1/2}A\Sigma^{-1/2}}$. We identify these two terms as the two instance-dependent factors which control the hardness of the value function estimation problem under the $L_2(\mu)$ norm. We next give an instance-dependent lower bound which shows that for any instance values of the two parameters (within certain domains), there is a nearly-matching lower bound on the achievable asymptotic approximation factor.
\begin{restatable}{theorem}{aliasedlocallb}\label{thm:aliased-local-lb}
In the aliased setting, $\forall\, x \in [1,\infty], \forall y \in (0,\frac{1}{2})$, there exists a collection of two instances $\bM=\{(\cM_1,\mu_1,\phi_1),(\cM_2,\mu_2,\phi_2)\}$ which both satisfy $\norm{\Pi_\mu P}_\mu = x$ and $\sigmamin{\Sigma^{-1/2}A\Sigma^{-1/2}}=y$ and generate the same data distribution $\bQ$, yet any estimator $\hat{v}$ will satisfy
\begin{equation}\label{eq:aliased-local-lb}
\sup_{(\cM,\mu,\phi) \in \bM} \alpha^{\hat{v}}_\mu(\cM,\mu,\phi) \geq \sqrt{ 1 + \gamma^2 \frac{\norm{\Pi_\mu P}^2_\mu-1}{\sigma^2_{\min}(\Sigma^{-1/2}A\Sigma^{-1/2})}} 
\end{equation}
\end{restatable}
\begin{proof}[Proof (sketch)]
	 We construct two MRPS which, under aliasing, will generate the same data distribution. However, the two MRPs have different value functions and one will be realizable. In particular, the approximation ratio will be infinite if learner doesn't output that particular value function. The lower bound is obtained by calculating the error of this value function as the estimate for the first MRP. See Figure \ref{fig:aliasing-lb-l2mu} for an illustration of the two MRPs, and Appendix \ref{app:aliased-local-lb} for a full proof. 
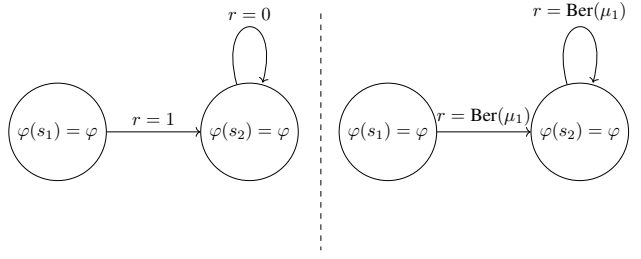
\begin{figure}
\scalebox{0.7}{
\begin{tikzpicture}
\node[state] (q1) {$\varphi(s_1) = \varphi$};
\node[state, right=0.7in of q1] (q2) {$\varphi(s_2)=\varphi$};
\node[state, right=0.3in of q2] (p1) {$\varphi(s_1) = \varphi$};
\node[state, right=0.7in of p1] (p2) {$\varphi(s_2)=\varphi$};
\draw[->] (p1) edge[above] node{$r=\text{Ber}(\mu_1)$} (p2)
(p2) edge[loop above] node{$r=\text{Ber}(\mu_1)$} (p2);
\draw[->] (q1) edge[above] node{$r=1$} (q2)
(q2) edge[loop above] node{$r=0$} (q2);
\draw [dashed] (5,2.25) -- (5,-2.25);
\end{tikzpicture}
}
\caption{The construction of Theorem \ref{thm:aliased-local-lb}. Left: MRP $\cM_1$. Right: MRP $\cM_2$. They generate the same aliased distribution $\bQ$.}\label{fig:aliasing-lb-l2mu}
\end{figure}
\end{proof}
The numerator in the second term of the lower bound is always non-negative due to the restriction on the domain of $x$. Furthermore, when $x > \sqrt{2}$, then the upper bound (Eq. \eqref{eq:lstd-ub}) and the lower bound (Eq. \eqref{eq:aliased-local-lb}) differ by at most a multiplicative factor of $2$. Thus, in this regime of the instance-dependent parameters, LSTD attains the asymptotically optimal approximation ratio up to constant factors. Our domain restrictions on $x$ and $y$ in the lower bound also do not preclude the interesting regimes of the problem, i.e. the cases where $\norm{\Pi_\mu P}_\mu$ is large ($\rightarrow \infty$) or $\sigmamin{\Sigma^{-1/2}A\Sigma^{-1/2}}$ is small ($\rightarrow 0$). Of course, this lower bound heavily relies on the aliased nature of the problem. %
Our next section examines whether the same lower bound holds in the non-aliased setting, where the learner is less restricted.

\subsection{Without aliasing: what is the optimal approximation factor?}\label{sec:l2mu-no-aliasing}

In the non-aliased case, the learner can  still use the LSTD algorithm, so the upper bound of Theorem \ref{thm:lstd-ub} still holds. For the lower bounds, the class of learners that we are competing against now have more information. 
We show that, despite that LSTD does not use the state information, it still attains the optimal approximation factor for this setting. As this result contains our most intricate argument, we start with two illustrative warm-up results showing that both of our instance-dependent factors appearing in Equation \eqref{eq:lstd-ub} are independently necessary, meaning that the finiteness of one alone does not guarantee a finite approximation ratio. The main result is given in Theorem \ref{thm:l2mulocal}.

\subsubsection{$\norm{\Pi_\mu P}_\mu$ is necessary}

The first result of two exhibits a family of instances where $\sigmamin{\Sigma^{-1/2}A\Sigma^{-1/2}} > 0$ yet the approximation ratio of any estimator is infinite. By the upper bound of Theorem \ref{thm:lstd-ub}, this must indicate that $\norm{\Pi_\mu P}_\mu = \infty$, and indeed this is the case. 

\begin{restatable}{lemma}{PiPnecessary}\label{lemma:PiPnecessary}
In the non-aliased setting, there exists a family of instances $\bM=\{(\cM,\mu,\phi)\}$ which all have an $L_2(\mu)$-misspecification of $0$, $\sigmamin{\Sigma^{-1/2}A\Sigma^{-1/2}} > 0$, and $\norm{\Pi P}_\mu = \infty$, yet any estimator $\hat{v}$ will satisfy 
\[
\sup_{(\cM,\mu,\phi) \in \bM} \alpha^{\hat{v}}_\mu(\cM,\mu,\phi) = \infty
\]
\end{restatable}

\begin{proof}%
We take MRPs which have the same transition dynamics as those in the construction of Theorem \ref{thm:aliased-local-lb}. They have a reward $r(s_1) = 0$ and $r(s_2) = r$, and $\mu(s_1) = 1$, $\mu(s_2)=0$. The features are arbitrary non-zero vectors. Realizability under $L_2(\mu)$ is trivially satisfied since only $\supp(\mu) = \{s_1\}$ so we only have to represent a scalar. No estimator can recover the true value function since there is no data on state $s_2$. See Appendix \ref{app:PiPnecessary} for a full proof. 
\end{proof}

This example illustrates the interpretation that $\norm{\Pi_\mu P}_\mu$ captures the intuitive source of hardness in value function estimation. One interpretation for the case where the features are ``tabular'' is that $\norm{\Pi_\mu P}_\mu$ is large (or infinite) when there is a lack of ``pushforward'' coverage \cite{xie2021batch}, meaning that a state $s \in \supp(\mu)$ may transition to a state $s' \notin \supp(\mu)$. Since the value at $s$ depends on the value at $s'$, we may not be able to predict $v_\cM(s)$ even under realizability. However, in the sequel we will see examples where there is no pushforward converage but $\norm{\Pi_\mu P}_\mu$ can still be bounded thanks to the features. Namely, our next result shows that, surprisingly, this is not the only source of hardness in the off-policy value estimation problem.

\subsubsection{$\sigma_{\min}(\Sigma^{-1/2}A\Sigma^{-1/2})$ is also necessary}

We next examine the case where $\norm{\Pi_\mu P}_\mu$ is finite but $\sigma_{\min}(\Sigma^{-1/2}A\Sigma^{-1/2})$ is zero. We show that the condition $\norm{\Pi_\mu P}_\mu < \infty$ implies a somewhat strong structure on the features (cf. Lemma \ref{lemma:pipmu}). Even in the presence of this condition, we show that one can find a set of instances where any estimator will have an infinite approximation ratio. The upper bound of Theorem \ref{thm:lstd-ub} implies that $\sigmamin{\Sigma^{-1/2}A\Sigma^{-1/2}}=0$ must be the case on these instances, and indeed this is the case. %
\begin{restatable}{theorem}{sigmaminlb}\label{thm:sigmamin-lb}
In the non-aliased setting, there exists a family of instances $\{(M,\mu,\phi)\}$ which all have an $L_2(\mu)$-misspecification of $0$, $\norm{\Pi_\mu P}_\mu < \infty$, and $\sigmamin{\Sigma^{-1/2}A\Sigma^{-1/2}} = 0$, yet any estimator $\hat{v}$ will satisfy 
$$
\sup_{(\cM,\mu,\phi) \in \bM} \alpha^{\hat{v}}_\mu(\cM,\mu,\phi) = \infty
$$
\end{restatable}
\begin{proof}[Proof (sketch)]
We pick a $5$-state MRP with $3$ $\mu$-supported states (numbered $1, 2, 3$) and $2$ $\mu$-unsupported states (numbered $4,5$). %
We set the reward to be zero except for $r(4)$ and $r(5)$ (which will be unknown to the learner). For a fixed transition matrix $P$, let $\bd = (I - \gamma P)^{-1}$ denotes its discounted occupancy matrix, and $\bd_4$ and $\bd_5$ denote the fourth and fifth columns of this matrix, respectively. Geometrically, the space of possible value functions which the learner must distinguish between is the $2$-dimensional plane $\cV_\cM \coloneqq \{r(4) \cdot \bd_4 + r(5) \cdot \bd_5\}_{r_4,r_5 \in [-1,1]}$. We then pick a $1$-dimensional feature map $\Phi = \lambda_1 \bd_4 + \lambda_2 \bd_5 \in \bR^{5 \times 1}$, which is a linear combination of the columns of $\bd$ and thus lies in the plane. Thus, there are an infinite number of realizable value functions, and the learner cannot distinguish the correct one without knowing $r(4)$ and $r(5)$ (which occur at unsupported states). This implies that the approximation ratio is infinite. The only thing left to check is that $\norm{\Pi_\mu P}_\mu < \infty$. In the presence of unsupported states,  this would imply the following structural condition. %
\begin{restatable}{lemma}{pipmu}\label{lemma:pipmu}
Under Assumption \ref{ass:lrassumptions}, $\norm{\Pi_\mu P}_\mu < \infty$ if and only if $\,\forall s' \notin \supp(\mu), \bE_{s \sim \mu}\left[\phi(s) \cP(s'|s)\right] =0$.
\end{restatable}

See Appendix \ref{app:sigmamin-lb} for a proof of Lemma \ref{lemma:pipmu}. This condition (along with the condition that $\supp(\mu) = \{1,2,3\}$) turns out to be a set of bilinear condition in both $\mu$ and $\Phi$, and we proceeded by random search to find a problem $(P, \lambda_1, \lambda_2, \mu(1),\mu(2),\mu(3))$ which satisfies this condition. See Appendix \ref{app:sigmamin-lb} for a full description of the MRP.
\end{proof}

\begin{figure}
\begin{center}
\includegraphics[scale=0.6]{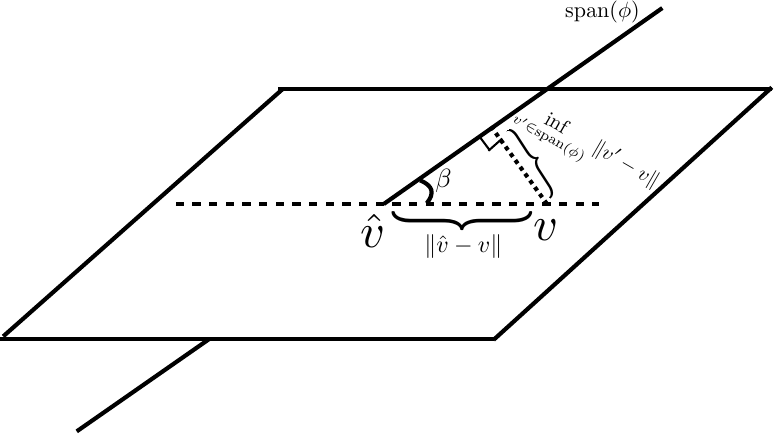}
\caption{Illustration of proof of \cref{thm:sigmamin-lb}. The plane is a  space of possible value functions $\cV_\cM$ which the nature can choose from without leaking more information to the learner. The line represents the space of possible linear predictors $\cF_\Phi$. The true value function is $v$, and the best estimator is $\hat{v}$. As shown on the figure, the angle $\beta$ determines the approximation ratio, and $\beta = 0$ (line in the plane) implies $\infty$ approximation ratio. In our construction, $\beta$ is controlled by the magnitude of $A$.}\label{fig:angledotpdf}
\end{center}
\end{figure}

Why is $A=0$ implied by the construction of the previous proof? While it may appear surprising that the invertibility of some algorithm-specific quantity (the $A$ matrix) can dictate the hardness of value function estimation for \textit{all} estimators, the intuition is that $A=0$ implies that the linear subspace $\cF_\Phi$ can live completely inside of the space of ``plausible'' value functions which the learner can not distinguish between ($\cV_\cM$, in the notation of our proof). More formally, when $\Phi$ is a linear combination of columns from the discounted occupancy matrix, we have that $D \Phi = D (\gamma P)\Phi$ and thus $A = 0$. In the general case where $A$ is nonzero, its minimum singular value dictates the ``angle'' between the $\cF_\Phi$ and $\cV_\cM$, and a small angle indicates a large approximation error (see Figure \ref{fig:angledotpdf}). In conclusion, we have showed that $\norm{\Pi_\mu P} < \infty$ and $\sigma_{\min}(\Sigma^{-1/2}A\Sigma^{-1/2}) > 0$ are \textit{both} independently necessary for finite approximation factors in value function estimation under the $L_2(\mu)$ norm. 

\subsubsection{LSTD attains the optimal approximation factor}\label{sec:l2munoalias}

We can combine all of our observations so far to establish a general lower bound for the non-aliased case.

\begin{restatable}{theorem}{Ltwomulocal}\label{thm:l2mulocal}
For all $x \in (0, \infty)$, there exists three problem instances $\bM = \{ (\cM_i,\mu_i,\phi_i)\}_{i=1}^3$ that all have $x = \norm{\Pi_\mu P}_\mu / \sigmamin{\Sigma^{1/2}A\Sigma^{1/2}}$, yet any estimator $\hat{v}$ will satisfy 
\begin{align*}
\max_{(\cM,\mu,\phi) \in \bM} \alpha_\mu^{\hat{v}}(\cM,\mu,\phi) &\geq  \frac{\norm{\Pi_\mu(I - \gamma P)}_\mu}{\sigma_{\min}(\Sigma^{-1/2}A\Sigma^{-1/2})} - 1 \\
&\geq \frac{\gamma \norm{\Pi_\mu P}_\mu - 1}{\sigma_{\min}(\Sigma^{-1/2}A\Sigma^{-1/2})} - 1
\end{align*}
\end{restatable}
\begin{proof}[Proof (sketch)]
    We use a similar idea to Theorem \ref{thm:sigmamin-lb}. We take a $5$-state MDP with $\supp(\mu) = \{1,2,3\}$ and unobserved rewards at $r(4)$ and $r(5)$. The features are slightly perturbed away from being a linear combination of columns from the discounted matrix. Namely, in the notation of Theorem \ref{thm:sigmamin-lb}, we take $\Phi = \alpha \bd_4 + \beta \bd_5 + \varepsilon \psi$, for some small $\varepsilon$ and carefully-chosen $\psi$. The size of $\varepsilon$ will control $\sigmamin{A}$, while the size of $\norm{\Pi_\mu P}_\mu$ will be controlled by $\mu$. To ensure that $\norm{\Pi_\mu P}_\mu < \infty$  (or Lemma \ref{lemma:pipmu}) continues to hold as we vary $\varepsilon$ and $\mu$, we derive a structural condition that relates this to the $(\alpha,\beta,\varepsilon)$ vector lying in the kernel of some matrix. We find $P$ which is rank-deficient but for which $\bd$ is full-rank. This ensures that the kernel of this matrix is non-trivial, which means that we can pick $(\alpha,\beta,\varepsilon)$ inside this kernel as a continuous function of the problem parameters while also ensuring that $\Phi \neq 0$. We also pick the perturbation vector $\psi$ to (roughly) realize the operator norm of $\Pi_\mu P$, i.e. it satisfies $\norm{\Pi_\mu P \psi}_\mu = \norm{\Pi_\mu P}_\mu$. Such a choice follows by a careful argument involving Brouwer's fixed point theorem (since $\Pi_\mu$ depends on $\psi$ via $\Phi$) and continuity of eigenvectors. See Appendix \ref{app:l2mulocal} for a full proof.
\end{proof}

In terms of instance-dependence, this result is slightly weaker than our previous bounds since the statement only allows control over the \textit{ratio} of the two instance-dependent terms. As before, the domain of $x$ is allowed to vary to capture the interesting asymptotics of the problem ($x \rightarrow \infty$). Once again we have that, for $x$ larger than some constant, the lower bound is within a constant factor of the upper bound. Specifically, when $x > 4$, the upper bound and lower bound differ by a factor of 2, and this ratio converges to $1$ as $x \rightarrow \infty$. This result establishes that, for the interesting regimes where the parameters are large enough, LSTD attains the optimal approximation ratio (up to constant factors) in the $L_2(\mu)$ norm in both the aliased and non-aliased cases.

\paragraph{Cases where the approximation factor is well-behaved} Even though the blowup from these two factors is unavoidable in general, along the way we identify several novel conditions under which the optimal value function can be recovered, either by LSTD or by alternative estimators. These are outlined in Appendix \ref{app:alphaequals1}. Particularly interesting conditions are when $P$ maps orthogonal value functions (i.e. functions not lying in the span of $\Phi$) to orthogonal value functions, or when $\norm{P}_\mu < \infty$ (noting that this is stronger than just $\norm{\Pi_\mu P}_\mu < \infty$, since $\norm{\Pi_\mu P}_\mu \leq \norm{\Pi_\mu}_\mu \norm{P}_\mu = \norm{P}_\mu$).

\section{Approximation Ratios in the $L_\infty$ Norm}\label{sec:linfty}

In this section we study the optimal asymptotic approximation factor for the $L_\infty$ norm. A summary of the results for this section can be found in the ``$L_\infty$ norm'' column of Table \ref{fig:table-results}. Recall that we write $\alpha_\infty$ for approximation factors in this norm. 

\subsection{LSTD attains the optimal approximation factor}\label{sec:linfty-lstd}

We begin with an upper bound for LSTD. We first note that it is possible (see Appendix \ref{app:translating}) to convert an approximation ratio bound for the $L_2(\mu)$ norm into an approximation ratio bound for the $L_\infty$ norm by paying an extra factor of $1/\lambda_{\min}(\Sigma)$ (which is finite by Assumption \ref{ass:lrassumptions}, but may be arbitrarily large). However, our next result shows that this eigenvalue dependence is not necessary and a more direct approach yields a better result. %

\begin{restatable}{theorem}{lstdubinfty}\label{thm:lstd-ub-infty}
Assume that the $A$ matrix from Equation \eqref{eq:A-matrix-def} is invertible. Then the population LSTD estimator has an approximation factor upper bound of 
\[
\alpha_\infty^{\texttt{LSTD}} \leq 1 + \norm{\Phi A^{-1} \Phi^\top D(I - \gamma P)}_\infty \leq 1 + \frac{1+\gamma}{\sigmamin{A}}
\]
\end{restatable}
\begin{proof}[Proof (sketch)]
Relies on another exact decomposition of the LSTD error:
\[
\Pi_\infty v_\cM - \Phi\LSTD = \Phi A^{-1}\Phi^\top D(I-\gamma P) (\Pi_\infty v_\cM - v_\cM).
\]
See Appendix \ref{app:lstd-ub-infty} for a full proof.
\end{proof}

This shows that, in the $L_\infty$ norm, the upper bound obtained by LSTD only depends on the minimum singular value of $A$ (rather than on the same singular value as well as $\norm{\Pi_\mu P}_\mu$, as was the case for the $L_2(\mu)$ norm). At first glance it might appear strange that there are \emph{less} problem-dependent factors in the $L_\infty$ bound (which should be a harder norm to minimize), but the resolution to this apparent contradiction is that a guarantee of small misspecification under the $L_\infty$ norm is a substantially stronger assumption. Intuitively, the usefulness of the $L_2(\mu)$ misspecification guarantee hinges on our ability to translate the misspecification error on $\mu$-supported states to other parts of the state space, which additionally depends on $\norm{\Pi_\mu P}_\mu$.

On the lower bound side, we can combine the ideas of the construction from \cite{amortila2020variant} and our previous lower bound (Theorem \ref{thm:sigmamin-lb}) to establish that LSTD attains the optimal approximation factor for regimes where $\gamma$ is large enough. Formally, the result is that: 

\begin{restatable}{theorem}{sigmamininfty}\label{thm:sigmamin-infty}
In the non-aliased setting, for all $\gamma \in [c_1,1)$ where $c_1$ is some absolute constant, and for all $y \in [0,1-\gamma]$, there exists three instances $\{(\cM_i,\mu_i,\phi_i)\}$ which all satisfy $\sigmamin{A} = y$ yet
\[
\inf_{\hat{v}} \sup_{(\cM_i,\mu_i,\phi_i)}\alpha_\infty^{\hat{v}}(\cM_i,\mu_i,\phi_i) \geq \frac{1}{2} + \frac{\gamma}{\sigmamin{A}}.
\]
The value of the constant is upper bounded by $c_1 \leq 0.7$.
\end{restatable}
\begin{proof}[Proof (sketch)]
    The proof uses the MRP construction from \citep{amortila2020variant} (and Lemma \ref{lemma:PiPnecessary}) but perturbs the features by adding a column of the discounted occupancy matrix (similar to the construction of Theorem \ref{thm:sigmamin-lb}). See Appendix \ref{app:sigmamin-infty} for a full proof. 
\end{proof}

We note again that the domain for our problem-dependent parameters ($y \in [0,1-\gamma]$) do not preclude the interesting regimes, which are when $\sigmamin{A} \rightarrow 0$. We also note that, since the lower bound holds for the non-aliased setting, it also holds for the (harder) aliased setting. Towards comparing the upper and lower bound, we can take their ratio, use the bounds on $\gamma$ and $y$, and observe that the ratio is always upper bounded by $2$. Thus, for this regime of problem parameters, the lower bounds and upper bounds match up to a constant factor. This lower bound does not make use of aliasing, while the upper bound applied to the aliased case, and thus for general $\mu$ the optimal approximation factor is not affected by the presence of aliasing.

\subsection{Optimal approximation ratio under full support}\label{sec:full-support}

In this section, we examine a natural additional assumption which enables an alternative model-based estimator that asymptotically achieves a much better approximation ratio of $(1-\gamma)^{-1}$, which is independent of $1/\sigmamin{A}$. On the other hand, this estimator will be much less sample-efficient, as its sample complexity will depend on the cardinality of $|\phi(\cS)|$. Formally, the assumption is: 
\begin{assumption}[Full support]\label{ass:full-support}
 	The off-policy distribution $\mu$ is such that $\supp(\mu) = \cS$.
\end{assumption}

The full support assumption appears somewhat commonly in the literature when $L_\infty$ norms are concerned \cite{huangbeyond,bertsekas1996neuro}. We note that Assumption \ref{ass:full-support} renders the problem trivial in the non-aliased setting as we can asymptotically recover the true value function $v_\cM$. However, it remains an interesting question whether a similar result is possible under aliasing. Our estimator is based on state abstractions. %

\paragraph{State abstractions}  We call $\phi(\cS) \coloneqq \cX$ the \emph{abstract space}, and we denote abstract states by $x,x'$. Note that $X \coloneqq |\cX| \leq |\cS|$ and in particular the abstract space is also finite. This estimator ignores the topology on $\cX$ and instead learns a pointwise function on the abstract space. Our estimator is defined as the solution to the Bayes model $M_\phi = (r_\phi,P_\phi)$, where: $r_\phi(x) = \bE[r(s) \mid \phi(s) = x] \in \bR^X$ and $
 P_\phi(x,x') = \bP(x' \mid x) \in \bR^{X \times X}$. Note that the Bayes model implicitly depends on the off-policy distribution $\mu$ via the condition expectations. The solution to this model is 
\begin{equation}\label{eq:abstactvphi}
v_\phi = (I - \gamma P_\phi)^{-1}r_\phi \in \bR^X,
\end{equation}
which we call the Bayes value function. The following result shows that $v_\phi$ has a well-behaved approximation ratio (see Appendix \ref{app:aliasing-ub-infty} for a proof). 
\begin{restatable}{theorem}{aliasingubinfty}\label{thm:aliasing-ub-infty}
Under Assumption \ref{ass:full-support}, the estimator $v_\phi$ from Equation \eqref{eq:abstactvphi} has an approximation ratio of $\frac{2}{1-\gamma}$, i.e. we have
\begin{align*}
\norm{v_\varphi \circ \phi - v_M }_\infty &\leq \frac{2}{1-\gamma}\inf_{f: \cX \mapsto \bR} \norm{ f \circ \varphi - v_\cM}_\infty \\
&\leq \frac{2}{1-\gamma}\inf_{\theta} \norm{ \Phi \theta - v_\cM}_\infty
\end{align*}
\end{restatable}

 Indeed, one can construct examples where this estimator is infinitely better than LSTD, by taking $\sigmamin{A} \rightarrow 0$, which causes the LSTD parameter to diverge. This is illustrated in the counterexample of \cite{kolter2011fixed}, where LSTD diverges but our Bayes estimator achieves an approximation ratio of $1$. Our next result shows that the approximation ratio $2/(1-\gamma)$ is arbitrarily close to optimal. 

\begin{restatable}{theorem}{aliasinglbinfty}\label{thm:aliasing-lb-infty}
In the aliased setting, under Assumption \ref{ass:full-support},  $\forall \varepsilon > 0$, $\forall\,\gamma \in (0,1)$, there exists a collection of two instances $\bM=\{(\cM_1,\mu_1,\phi_1),(\cM_2,\mu_2,\phi_2)\}$ which generate the same data distribution $\bQ$, yet any estimator $\hat{v}$ will satisfy
\[
\sup_{(\cM,\mu,\phi) \in \bM} \alpha^{\hat{v}}_\infty(\cM,\mu,\phi) \geq \frac{2}{1-\gamma} - \varepsilon
\]
\end{restatable}
\begin{proof}[Proof (sketch)]
We use same construction as Theorem \ref{thm:aliased-local-lb}, but the error remains bounded when we are under the $L_\infty$ norm. See Appendix \ref{app:aliasting-lb-infty} for a full proof. 
\end{proof}

It is interesting to note that this abstract model-based estimator does not work under $L_2(\mu)$ misspecification. In particular, the construction in the lower bound of Theorem \ref{thm:aliased-local-lb} satisfies the full-support assumption (Assumption \ref{ass:full-support}), yet the minimax $L_2(\mu)$ error can be taken to infinity by taking the ``pushforward'' parameter $\norm{\Pi_\mu P}_\mu \approx \mu(s_1)/\mu(s_2) \rightarrow \infty$. We note that the function $v_\phi \circ \phi$ may not be a linear function of the features (since it is defined pointwise for each value of $\phi$). We can output a linear function simply by taking the $L_\infty$ projection to the set of linear functions, which results in a final bound of $1+\frac{2}{1-\gamma}$ (see  Corollary \ref{cor:projected-v-phi} in Appendix \ref{app:projected-v-phi}). 

\section{Related works}\label{sec:related}

\paragraph{Existing negative results for off-policy evaluation}

With finite-horizons MRPs, \cite{wang2020statistical} show that the sample complexity of off-policy evaluation problem under realizability and a good $\lambda_{\min}(\Sigma)$ may be exponential in $d$ or in $H$, the horizon. This was adapted to the infinite-horizon setting by \citep{amortila2020variant} which shows that even with $L_\infty$ realizability and good $\lambda_{\min}(\Sigma)$ the true solution may be asymptotically unidentifiable. \citeauthor{perdomo_sharp_2022} extend this to show that the invertibility of $A$ may be necessary for identifiability in the realizable setting:  they show that amongst a class of ``linear estimators'' (which depend only on certain first-moment quantities), \emph{any} instance where $A=0$ can be modified such that these estimators cannot recover the true value function. Their lower bound only applies to a restricted class of estimators, whereas ours rules out all estimators. However, the ``instance-dependence'' in their result is stronger as it allows for arbitrary instances $(\cM,\mu,\phi)$; we instead allow for arbitrary parameter values and then take a worst-case over instances with these values. %
In the misspecified on-policy case, \citep{mou_optimal_2020} show that LSTD has the optimal approximation factor for restricted sample sizes $n$ satisfying $n^2 + d \lesssim S$. In the on-policy case, the asymptotic approximation ratio is $1$, so the hardness in their result comes from the sample size restriction, whereas ours comes from the non-stationarity of the off-policy distribution $\mu$. Overall, there was no precise understanding of when this problem is solvable/not solvable (a result which is captured by our instance-dependent bounds), or which blowup is optimal in the off-policy case (even asymptotically).

\paragraph{Existing guarantees for LSTD}

The LSTD algorithm was originally proposed by \citep{bradtke1996linear}. There have been several sample complexity analyses, (e.g. \citet{perdomo_sharp_2022, pires2012statistical,tu2018least,duan_optimal_2021}). In terms of approximation ratios under misspecification, in the on-policy case, \citep{tsitsiklis_analysis_1997} derive the classical approximation ratio bound of $(1-\gamma^2)^{-1/2}$, which uses the fact that $P$ (and thus $\Pi_\mu P$) are contractive in the $L_2(\mu)$ norm when $\mu$ is the stationary distribution. %
This bound was sharpened in an instance-dependent fashion by \citet{yu_error_2010, mou_optimal_2020}, which both consider the more general problem of solving projected fixed point equations. Their approximation bounds are similar to our Theorem \ref{thm:lstd-ub}, although our proof relies on a simpler and exact error decomposition. Our proof also enables us to readily derive $L_\infty$ bounds, whereas only $L_2(\mu)$ bounds are considered the above works. Conversely, the work of \citep{perdomo_sharp_2022} provides approximation bounds only in the $L_\infty$ norm, and these sub-optimally scale with \textit{both} $\lambda_{\min}(\Sigma)$ and $\sigmamin{A}$ (similar to our result in Appendix \ref{app:translating}). 
The work of \citep{scherrer2010should} studies both LSTD and Bellman Residual Minimization and shows that they are both instances of oblique projections onto certain subspaces, a perspective which yields the approximation factor of $\norm{\Pi_{(I-\gamma P)^\top D \Phi}}_\mu$, where $\Pi_X = \Phi (X^\top \Phi)^{-1} X^\top$ is the oblique projection operator.

\paragraph{OPE at large}
In the paper we show the quantities $\norm{\Pi_\mu P}_\mu$ and $\sigma_{\min}(\Sigma^{-1/2}A\Sigma^{-1/2})$ are both necessary for the $L_2(\mu)$ norm, and that $\sigmamin{A}$ is necessary for the $L_\infty$ norm. This implies that removing the finiteness of any of these quantities leads to unbounded approximation ratio. However, our results do not exclude the possibility that one can come up with alternative assumptions/quantities to replace them. In fact, there are two sets of alternative assumptions that are widely used in the OPE literature: (1) ``Bellman-completeness'' \citep{antos2008learning, munos2007performance, chen2019information, duan_minimax-optimal_2020}, which asserts that the function class is \textit{closed} under the Bellman operator, and (2) the realizability of so-called importance weight functions \citep{liu2018breaking,uehara_minimax_2020,miyaguchi2021asymptotically}. However, most of these works focus on the estimation of the expected return at the initial state distribution instead of recovering the full function (with \citet{huangbeyond} as an exception), and none of them study the optimality of the approximation ratio. Moreover, under these different assumptions, the definition of misspecification error changes (e.g., the violation of Bellman-completeness is sometimes referred to as ``inherent Bellman error'' (IBE) \citep{antos2008learning}), and so does the behavior of the approximation ratio. Under the $L_2(\mu)$ norm, small misspecification and small IBE do not imply each other, so studying the approximation ratio under these assumptions would be an interesting future direction. We can also compare the $\norm{\Pi_\mu P}_\mu$ quantity, a measure of data coverage, with the more classical notion of concentrability \citep{antos2008learning}.\footnote{In this setting we could for example define concentrability as $\norm{\rho/\mu}_\infty$ or $\norm{d/\mu}_\infty$, where $\rho$ is the stationary distribution and $d$ is the discounted state occupancy.} However, the construction of Theorem \ref{thm:sigmamin-lb} shows that $\norm{\Pi_\mu P}_\mu$ can be small while concentrability could be infinite, so this notion is too loose for our purposes. On the other hand, notions of \textit{linear concentrability}\footnote{For instance defined as the ratio of the operator norms for $\Sigma_\rho = \bE_\rho [\phi \phi^\top]$ and $\Sigma$, i.e. $\norm{\Sigma_\rho}_2/\norm{\Sigma}_2$} \cite{uehara2021representation}, are too weak to capture the hardness of this problem, as they can be made small in Lemma \ref{lemma:PiPnecessary} while keeping $\sigmamin{\Sigma^{-1/2}A\Sigma^{-1/2}} > 0$ and $\norm{\Pi_\mu P}_\mu = \alpha^\star = \infty$.

\section{Conclusion}

In this work we have highlighted the importance of understanding the necessary blowups to the approximation factors which occur in misspecified RL problems. We have posed a simple but fundamental learning problem, that of linear off-policy value function estimation, and focused on establishing the optimal approximation ratios for this problem achieved even by asymptotic estimators. We have provided instance-dependent upper and lowers bounds for a variety of settings (the $L_2(\mu)$ and $L_\infty$ norms, aliased and non-aliased observations, partial support and full-support). It was found that, for most of these settings, the LSTD estimator attains the optimal approximation factor and is fundamental in the sense that if it diverges that so does any other estimator. In other settings (namely, when there is full support), it was found that alternative model-based estimators instead achieve the optimal approximation ratios. 

While we have only been concerned with policy evaluation, it would also be important to consider misspecification in the more complicated offline policy optimization or online exploration problems. Here, there are even more hardness results which preclude a simple answer to this question \cite{lattimore2020learning,du2019good,weisz2021exponential,foster2021offline}.

\section*{Acknowledgements}
We thank Ashwin Pananjady and Akshay Krishnamurthy for helpful discussions. PA gratefully acknowledges funding from the Natural Sciences and Engineering Research Council (NSERC) of Canada. NJ acknowledges funding support from NSF IIS-2112471 and NSF CAREER IIS-2141781. CS gratefully acknowledges funding from NSERC and the Canada CIFAR AI Chairs Program through Amii.

\bibliography{misspec-ope1-2}
\bibliographystyle{icml2023}

\newpage
\appendix
\onecolumn
\section{Proofs for Section \ref{sec:l2mu}}

\subsection{Proof of Theorem \ref{thm:lstd-ub}}\label{app:lstd-ub}

\lstdub*

The proof follows from the following exact characterization of LSTD.
\begin{lemma}\label{lemma:lstd-ls}
If $A^{-1}$ exists, then we have that
$$
\LS - \LSTD = \gamma A^{-1} \Phi^\top D P v^\perp,
$$
where $v^\perp = v_\cM - \Pi_\mu v_\cM$. 
\end{lemma}
\begin{proof}
We have that $v_\cM = \Pi_\mu v_\cM + v^\perp$, with $\Pi_\mu v \in \col(\Phi)$ and $v^\perp \in (\col(\Phi))^\perp$ (note: the $\perp$ subspace is with respect to the $\mu$-weighted inner product). This equivalently means that $v^\perp \in \Ker(\Phi^\top D)$. Then we have: %
\begin{align*}
    (I - \gamma P) v_M &= r \\
    (I - \gamma P) \Phi \LS + (I- \gamma P) v^\perp &= r \\
    \Phi^\top D (I - \gamma P) \Phi \LS + \Phi^\top D (I - \gamma P) v^\perp &= \Phi^\top D r \tag{$\Phi^\top D$ on both sides}\\
    \left(\Phi^\top D(\Phi - \gamma P \Phi)\right) \LS - \gamma \Phi^\top D P v^\perp &= \Phi^\top D r \tag{$v^\perp \in \Ker(\Phi^\top D)$}\\
     A \LS  &= b + \gamma \Phi^\top D P v^\perp \tag{Defns of $A,b$}\\
     \LS &= A^{-1} b + \gamma A^{-1}\Phi^\top D P v^\perp \tag{$A^{-1}$ exists}
\end{align*}
Meanwhile, the LSTD solution is defined by $\LSTD = A^{-1}b$. Substracting both of these gives:
 \begin{equation}\label{eq:ls-lstd}
 \LS - \LSTD = \gamma A^{-1} \Phi^\top D P v^\perp.
 \end{equation}
 
 Note that we also have: 
 
  \begin{equation}\label{eq:ls-lstd-2}
 \LS - \LSTD = - A^{-1} \Phi^\top D (I - \gamma P) v^\perp,
 \end{equation}

since $\Phi^\top D v^\perp = 0$.
 
\end{proof}

\begin{corollary}
Let $\vLSTD = \Phi \LSTD$ and $\vLS = \Phi \LS$. Then we have the two inequalities
\[
\norm{\Phi \LSTD - \Phi \LS}_\mu = \gamma \norm{\Phi A^{-1} \Phi^\top D P v^\perp}_\mu \leq \gamma \norm{\Phi A^{-1} \Phi^\top D P}_\mu \norm{v^\perp}_\mu
\]
and
\begin{equation}\label{eq:first-approx-ub}
\norm{\vLSTD - \vLS}_\mu \leq \frac{\gamma}{\sigma_{\min}(I - \gamma \Sigma^{-1/2}\Sigmacr\Sigma^{-1/2})}\norm{\Pi_\mu P v^\perp}_\mu \leq \frac{\gamma}{\sigma_{\min}(\Sigma^{-1/2}A\Sigma^{-1/2})}\norm{\Pi_\mu P}_\mu \norm{v^\perp}_\mu,
\end{equation}
where $\Sigma = \Phi^\top D \Phi = \bE_{\mu}[\phi(s) \phi(s)^\top]$ is the covariance matrix and $\Sigmacr = \Phi^\top D P \Phi = \bE_{\mu,P}[\phi(s) \phi(s')^\top]$ is the cross-covariance.
\end{corollary}
\begin{proof}
The first equality follows from 
$$
\norm{\Phi \LSTD - \Phi \LS}_\mu = \gamma \norm{\Phi A^{-1} \Phi^\top D P v^\perp}_\mu \leq \gamma \norm{\Phi A^{-1} \Phi^\top D P}_\mu \norm{v^\perp}_\mu
$$
The second equality follows from:
\begin{align}
    \norm{\Phi \LSTD - \Phi \LS}_\mu &= \norm{\Sigma^{1/2}(\LSTD - \LS)}_2 \nonumber \\
    &= \gamma \norm{\Sigma^{1/2} A^{-1}\Phi^\top D P v^\perp}_2 \nonumber  \\
    &= \gamma \norm{(I - \gamma \Sigma^{-1/2} \Sigmacr\Sigma^{-1/2})^{-1}\Sigma^{-1/2}\Phi^\top D P v^\perp}_2 \nonumber \\
    &\leq \frac{\gamma}{\sigma_{\min}(I - \gamma \Sigma^{-1/2}\Sigmacr\Sigma^{-1/2})}\norm{\Sigma^{-1/2}\Phi^\top D P v^\perp}_2 \nonumber  \\
    &= \frac{\gamma}{\sigma_{\min}(I - \gamma \Sigma^{-1/2}\Sigmacr\Sigma^{-1/2})}\norm{\Sigma^{1/2}\Sigma^{-1}\Phi^\top D P v^\perp}_2 \nonumber  \\
    &= \frac{\gamma}{\sigma_{\min}(I - \gamma \Sigma^{-1/2}\Sigmacr\Sigma^{-1/2})}\norm{\Sigma^{1/2}(\Phi^\top D \Phi)^{-1}\Phi^\top D P v^\perp}_2 \nonumber  \\
    &=\frac{\gamma}{\sigma_{\min}(I - \gamma \Sigma^{-1/2}\Sigmacr\Sigma^{-1/2})}\norm{\Phi(\Phi^\top D \Phi)^{-1}\Phi^\top D P v^\perp}_\mu \nonumber \\
    &= \frac{\gamma}{\sigma_{\min}(I - \gamma \Sigma^{-1/2}\Sigmacr\Sigma^{-1/2})}\norm{\Pi_\mu P v^\perp}_\mu  \label{eq:first-bound}
\end{align}
And then note that $I - \gamma \Sigma^{-1/2}\Sigmacr\Sigma^{-1/2} = \Sigma^{-1/2}(\Sigma - \gamma \Sigmacr)\Sigma^{-1/2} = \Sigma^{-1/2}A\Sigma^{-1/2}.$
\end{proof}

To conclude the proof of Theorem \ref{thm:lstd-ub}, we can use the Pythagorean theorem on $\vLSTD - \Pi_\mu v_\cM \in \col(\Phi)$ and $\Pi_\mu v_\cM - v_\cM \in (\col(\Phi))^\perp$:
\begin{align*}
\norm{\vLSTD - v_\cM}_\mu &= \sqrt{ \norm{\Pi_\mu v_\cM - v_\cM}_\mu^2 + \norm{\vLSTD - \Pi_\mu v_\cM}^2_\mu} \\
	&\leq \sqrt{ \norm{\Pi_\mu v_\cM - v_\cM}_\mu^2 + \left(\gamma \norm{\Phi A^{-1} \Phi^\top D P}_\mu\right)^2\norm{\Pi_\mu v_\cM - v_\cM}_\mu^2} \\
	&= \sqrt{ 1 + \left(\gamma \norm{\Phi A^{-1} \Phi^\top D P}_\mu\right)^2} \norm{\Pi_\mu v_\cM - v_\cM}_\mu \\
	&\leq \sqrt{ 1 + \left(\gamma \frac{\norm{\Pi_\mu P}_\mu}{\sigmamin{\Sigma^{-1/2}A\Sigma^{-1/2}}}\right)^2} \norm{\Pi_\mu v_\cM - v_\cM}_\mu \\
	&\leq \left(1 + \gamma\frac{\norm{\Pi_\mu P}_\mu}{\sigmamin{\Sigma^{-1/2}A\Sigma^{-1/2}}}\right)\norm{\Pi_\mu v_\cM - v_\cM}_\mu \tag{$\sqrt{1+x^2} \leq 1 +x$ whenever $x \geq 0$}
\end{align*}

\subsection{Proof of Theorem \ref{thm:aliased-local-lb}}\label{app:aliased-local-lb}

\aliasedlocallb*

\begin{proof}
Our first instance is $(\cM_1,\mu,\phi)$ defined via

\begin{center}
\begin{tikzpicture}
\node[state] (q1) {$\varphi(s_1) = \varphi$};
\node[state, right=1in of q1] (q2) {$\varphi(s_2)=\varphi$};
\draw[->] (q1) edge[above] node{$r=1$} (q2)
(q2) edge[loop above] node{$r=0$} (q2);
\end{tikzpicture}\label{fig:aliasing-lb-infty}
\end{center}

We set $\phi=1 \in \bR^1$, and define the shorthands $\mu_1 = \mu(s_1)$, $\mu_2 = \mu(s_2)$. The values of $\gamma$ and $\mu_1,\mu_2$ will be picked to ensure that $\sigmamin{\Sigma^{-1/2}A\Sigma^{-1/2}} = y$ and $\norm{\Pi_\mu P}_\mu = x$.

By taking the derivative of $\mu_1 (\theta - 1)^2 + \mu_2 (\theta)^2$ wrt $\theta$, the optimal estimator is $\theta_1 = \theta \phi = \mu_1$. It has a square error: 
$$
\norm{\theta_1\phi(s) - v_{\cM_1}}_\mu^2 = \mu_1 (\mu_1 -1)^2 + \mu_2(\mu_1)^2 = \mu_1-\mu_1^2 = \mu_1\mu_2
$$

Note that this instance has $
\bP = \begin{pmatrix}
0 & 1 \\
0 & 1 \\
\end{pmatrix}
$, $\Sigma^{-1} = 1$, and $\Pi_\mu = \Phi \Phi^\top D = (1, 1) (1, 1)^\top D =  \begin{pmatrix}
 1 & 1 \\
 1 & 1 \\	
 \end{pmatrix}
\begin{pmatrix}
\mu_1 & 0 \\
0 & \mu_2 \\	
\end{pmatrix}
= 
\begin{pmatrix}
\mu_1 & \mu_2 \\
\mu_1 & \mu_2 \\
\end{pmatrix}$. This gives $\Pi_\mu P = \begin{pmatrix}
\mu_1 & \mu_2 \\
\mu_1 & \mu_2 \\
\end{pmatrix} \begin{pmatrix}
0 & 1 \\
0 & 1 \\
\end{pmatrix} = \begin{pmatrix}
0 & 1 \\
0 & 1 \\
\end{pmatrix}$. The operator norm thus has a value
$$
\max_{\norm{v}_\mu=1} \norm {\Pi_\mu Pv}_\mu = \max_{\norm{v}_\mu = 1} \sqrt{\mu_1 v_2^2 + \mu_2 v_2^2},
$$
which is maximized by taking $v_1=0,v_2 = 1/\sqrt{\mu_2}$, giving a value of
$$
\norm{ \Pi_\mu P}_\mu = \sqrt{\frac{\mu_1}{\mu_2} + 1}. 
$$
Note that $\norm{\Pi_\mu P}_\mu \in [1,\infty]$. We need this to equal $x$ which is easily achieved by solving $1 + \frac{\mu_1}{1-\mu_1} = x^2 \implies \mu_1 = \frac{x^2-1}{x^2}$ which lies inside $(0,1)$ for all $x \in (1,\infty)$. The cases where $x=1$ or $x=\infty$ are handled by picking $\mu_1=1$ or $\mu_1=0$, respectively.
Meanwhile we also have that 
$$
A = \phi^2 - \gamma \phi^2 = \phi^2(1-\gamma) = (1-\gamma),
$$
and $\Sigma^{-1/2}A\Sigma^{-1/2} = A$. We need $A = y$, which is achieved by picking $\gamma= 1 - y$. Note the restriction on the domain of $y \in (0,1/2)$ means that $1/2 < \gamma < 1 $.
 
The second MRP is the one defined as:

\begin{center}
\begin{tikzpicture}
\node[state] (q1) {$\varphi(s_1) = \varphi$};
\node[state, right=1in of q1] (q2) {$\varphi(s_2)=\varphi$};
\draw[->] (q1) edge[above] node{$r=\text{Ber}(\mu_1)$} (q2)
(q2) edge[loop above] node{$r=\text{Ber}(\mu_1)$} (q2);
\end{tikzpicture}\label{fig:aliasing-lb-infty}
\end{center}

where again $\phi=1 \in \bR^1$. We take $\mu_1$ and $\mu_2$ to be the same as in the first MRP. This instance also has $A = \phi^2(1-\gamma) = (1-\gamma)$ and $\norm{\Pi_\mu P}_\mu = 1+\frac{\mu_1}{\mu_2}$, which is easily seen since the features and the transition dynamics are the same. Further note that these two MRPs generate the same aliased distribution $\bQ_{\cM,\mu,\phi}$ since they both generate $(\phi,0,\phi)$ with probability $1-\mu_1$ and $(\phi,1,\phi)$ with probability $\mu_1$.

The optimal estimator for $\cM_2$ is evidently $\theta_2 = \phi \theta = \mu_1/(1-\gamma)$, since $v_{\cM_2}(s_1) = v_{\cM_2}(s_2) = \mu_1/(1-\gamma)$. In particular, this second MRP is realizable so this forces the estimator to pick $\mu_1/(1-\gamma)$ when faced against these two examples ($\mu_1$ is known by looking at the occurrence of the triples $(\phi,1,\phi)$ in $\bQ_{\cM,\mu,\phi}$). And other choice of estimator will in fact have a worst-case approximation ratio $\sup_{\cM \in \{\cM_1,\cM_2\}} \alpha = \infty$. On the first instance, this estimator will have a squared error
\begin{align*}
\norm{\theta_2 \phi(s) - v_{\cM_1}}^2_\mu = \mu_1(\frac{\mu_1}{1-\gamma}-1)^2 + \mu_2(\frac{\mu_1}{1-\gamma})^2 &= \mu_1 ((\frac{\mu_1}{1-\gamma})^2 - 2\frac{\mu_1}{1-\gamma} + 1) + \mu_2 (\frac{\mu_1}{1-\gamma})^2 \\
&= (\frac{\mu_1}{1-\gamma})^2 (\mu_1+\mu_2) - 2\frac{\mu_1^2}{1-\gamma} + \mu_1 \\
&= (\frac{\mu_1}{1-\gamma})^2 - 2\frac{\mu_1^2}{1-\gamma} + \mu_1
\end{align*}

Taking the ratio of squared errors gives:

\begin{align*}
\frac{\norm{\theta_2 \phi(s) - v_{\cM_1}}_\mu^2}{\norm{\theta_1 \phi(s) - v_{\cM_1}}_\mu^2 } = \frac{(\frac{\mu_1}{1-\gamma})^2 - 2\frac{\mu_1^2}{1-\gamma} + \mu_1}{\mu_1-\mu_1^2} &= 
\frac{\frac{\mu_1}{(1-\gamma)^2} - 2\frac{\mu_1}{1-\gamma} + 1}{1-\mu_1}
\\
&= \frac{\frac{\mu_1}{(1-\gamma)^2} - 2\frac{\mu_1}{1-\gamma} + 1}{\mu_2} \\
&\geq 1 + \frac{\mu_1}{\mu_2} \left( \frac{2\gamma - 1}{(1-\gamma)^2} \right) \tag{$1/\mu_2 \geq 1$, and algebra} \\
&\geq 1 + \frac{\mu_1}{\mu_2} \left( \frac{1}{(1-\gamma)^2} \right) \tag{$\frac{1}{2} \leq \gamma \leq 1$} \\
&= 1 + \frac{ \norm{\Pi_\mu P}^2_\mu - 1 }{\sigmamin{\Sigma^{-1/2}A\Sigma^{-1/2}}^2} \\
&\geq 1 + \gamma ^2 \frac{ \norm{\Pi_\mu P}^2_\mu - 1 }{\sigmamin{\Sigma^{-1/2}A\Sigma^{-1/2}}^2} \\
\end{align*}

The LHS was the ratio of squared errors, so taking square roots gives $\alpha_\mu$ and the desired bound. 

\end{proof}

\subsection{Proof of Theorem \ref{lemma:PiPnecessary}}\label{app:PiPnecessary}

\PiPnecessary*

\begin{proof}
This example is a slight modification of the two-state example \cite{amortila2020variant}. See Figure \ref{fig:PiPnecessary}.

\begin{figure}[htp]\label{fig:PiPnecessary}
    \begin{center}
    \begin{tikzpicture}
    \node[state] (q1) {$\varphi(s_1) = \gamma$};
    \node[state, right=1in of q1] (q2) {$\varphi(s_2)=1+\varepsilon$};
    \draw[->] (q1) edge[above] node{$r=0$} (q2)
    (q2) edge[loop above] node{$r$} (q2);
    \end{tikzpicture}\label{fig:aliasing-lb-infty}
    \end{center}
    \caption{The construction of Lemma \ref{lemma:PiPnecessary}}
    \label{fig:eps-discounted}
\end{figure}
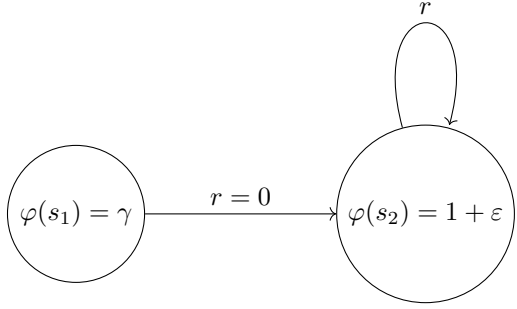
In the construction, we have $\mu(s_A) = 1$ and $\mu(s_B) = 0$. Note that $A = -\gamma^2\varepsilon \neq 0$.  Since the reward at state $s_B$ is never observed, any estimator will have constant error $\Omega(1/(1-\gamma))$ error asymptotically. Since $v_\cM$ is realizable on $s_A$ (with $\theta = r/(1-\gamma)$), the misspecification is $0$. Thus, the approximation ratio is infinite. The last thing to show is that $\norm{\Pi_\mu P}_\mu = \infty$. This is because $
P = \begin{pmatrix}
0 & 1 \\
0 & 1 \\
\end{pmatrix}
$, and $\Pi_\mu = \begin{pmatrix}
1 & 0 \\
(1+\varepsilon)/\gamma & 0 \\
\end{pmatrix}$ so $\Pi_\mu P = \begin{pmatrix}  
0 & 1 \\
0 & (1+\varepsilon)/\gamma \\
\end{pmatrix}$ thus $\norm{\Pi_\mu P}_\mu = \max_{\norm{v}_\mu=1} \norm{\Pi_\mu P v} = \max_{\norm{v}_\mu=1} \norm{(v_2, \frac{1+\varepsilon}{\gamma} v_2)^\top}_\mu = \max_{\norm{v}_\mu=1} v_2 = \infty$. In the last step we can take $v_2 \rightarrow \infty$ in the maximization since that state is unsupported. 
\end{proof}

\subsection{Proof of Theorem \ref{thm:sigmamin-lb}}\label{app:sigmamin-lb}

 \paragraph{Block matrix notation} In this section we will use the following convenient block matrix notation. Noting that $|\supp(\mu)| \leq S$ (with equality iff $\mu$ has support on all the states), we will re-arrange the states such that those that are supported are numbered $1 \dots |\supp(\mu)|$, and the unsupported ones are numbered $|\supp(\mu)|+1 \dots S$. Furthermore, for a given vector $v \in \bR^S$, we will write $v_\mu = (v_1, \dots, v_{|\supp(\mu)|})$ for the restriction of $v$ to the support states of $\cS$, and $v_{\neg \mu} = (v_{|\supp(\mu)| + 1}, \dots, v_S)$ for the restriction of $v$ to the unsupported states. Similarly, for a given matrix $X$, we will write it in block form as 
 $$
 X = \left[
\begin{array}{c|c}
X_{\mu,\mu}  & X_{\mu,\neg \mu}\\ \hline
 X_{\neg \mu, \mu} & X_{\neg \mu, \neg \mu}
\end{array}\right]
 $$

\sigmaminlb* 

We start by noting the following property.
\pipmu*
\begin{proof}
We show that $\norm{\Pi_\mu P }_\mu < \infty \iff (\Pi P)_{\mu, \neg \mu} = 0_{\mu, \neg \mu}$, and then that this implies $\forall i \in \mu, k \notin \mu$, $\langle \phi_i, \Sigma^{-1} \left(\sum_j \mu_j \phi_j P_{j,k} \right) \rangle = 0$. Lastly we show that If $\lambda_{\min}(\Sigma) > 0$ then $\norm{\Pi_\mu P }_\mu < \infty$ if and only if $\sum_j \mu \phi_j P_{j,k} =0 \, \forall k \notin \mu$.

The first part is easily observed by noting that $\norm{\Pi_\mu P}_\mu < \infty \iff \max_{\norm{v}_\mu = 1} \norm{\Pi_\mu P v}_\mu < \infty \iff \max_{\norm{v}_\mu = 1} \norm{( (\Pi_\mu P)_{\mu,\mu} v_\mu + (\Pi_\mu P)_{\mu, \neg \mu} v_{\neg \mu}; 0_{\neg \mu})}_\mu < \infty \iff (\Pi_\mu P)_{\mu, \neg \mu} = 0$, where the last line follows since if it has a non-trivial kernel then we can take $v_{\neg \mu}$ going to infinity while satisfying the constraints $\norm{v}_\mu = 1$. The second part is observed by expanding the definition of $(\Pi_\mu P)_{i,k}$ for all $i \in \mu$ and all $k \notin \mu$. For the last part, we note that $\lambda_{\min}(\Sigma) > 0$ implies that the span of $\{\phi(s_i)\}_{i \in \supp(\mu)} = \bR^d$. Thus, for each $k \notin \mu$, the set of equations $\langle \phi_i, \Sigma^{-1} \left(\sum_j \mu_j \phi_j P_{j,k} \right) \rangle = 0$ obtained by varying over all $i \in \mu$ must imply that the vector on the RHS must be $0$, i.e. $\Sigma^{-1} \left(\sum_j \mu_j \phi_j P_{j,k} \right) = 0 \implies \sum_j \mu_j \phi_j P_{j,k} = 0$ for each $k$. 
\end{proof}

\begin{proof}[Proof (of \ref{thm:sigmamin-lb})]
There are $m \coloneqq 3$ states in $\mu$, and $n\coloneqq 2$ states in $\neg \mu$. We number the known states as $1,2,3$ and the unknown states as $4$ and $5$. The states within $\mu$ transition amongst each other and to the unknown states. The unknown states simply self-loop. The reward will be 
$$
\bR = ( 0, 0, 0, r4, r5),
$$
where $r4$ and $r5$ are chosen later. We also set 
$$
\gamma = 9/10.
$$

The only things left to choose are $(\bP,\phi,\mu)$. Let us write down the transition matrix. 

$$
\bP = \begin{pmatrix}
0.313 & 0.2322 & 0.2999 & 0.0786 & 0.0763\\
0.8483 & 0.0014 & 0.0867 & 0.0484 & 0.0152\\
0.1144 & 0.2852 & 0.219 & 0.2437 & 0.1377\\
0 & 0 & 0 & 1 & 0\\
0 & 0 & 0 & 0 & 1\\
\end{pmatrix}
$$
where the floating point numbers are exact (i.e. can be represented as rationals). 

And of course the discounted occupancy matrix is 
$$
\mathbbm{d} \coloneqq (I - \gamma \bP)^{-1} = \begin{pmatrix}
2.22637 & 0.675069 & 0.814047 & 3.65445 & 2.63005\\
1.76839 & 1.56311 & 0.74639 & 3.56891 & 2.35319\\
0.85084 & 0.586281 & 1.58849 & 4.3413 & 2.63309\\
0 & 0 & 0 & 10 & 0\\
0 & 0 & 0 & 0 & 10\\
\end{pmatrix}
$$

The data distribution is not yet chosen but will have the following constraints
$$
\mu_1 >0, \mu_2 >0, \mu_3>0, \mu_4 = \mu_5 = 0
$$

The task of the learner is to predict a value function on $\mu$, i.e. on the first 3 states. Let us write $\mathbbm{d}_4$ for the $4^\text{th}$ column of $\mathbbm{d}$, and $\mathbbm{d}_{5}$ for the $5^\text{th}$ of $\mathbbm{d}$. The space of possible value functions in this MRP is 
$$
\cV_M = \left\{ v = r4\cdot \mathbbm{d}_{4} + r5\cdot \mathbbm{d}_{5} \mid r4,r5 \in [-1,1]  \right\} \subseteq \bR^5,
$$
since we set $r_1=r_2=r_3=0$. The space of possible value functions restricted to $\mu$ is:
$$
\cV^\mu_M = \left\{ (v(s_1),v(s_2),v(s_3))^\top = r4\cdot \mathbbm{d}_{1::3,4} + r5\cdot \mathbbm{d}_{1::3,5} \mid r4,r5 \in [-1,1]  \right\} \subseteq \bR^3,
$$
where $\mathbbm{d}_{1::3,4}$ is the first $3$ elements of $\mathbbm{d}_4$, and $\mathbbm{d}_{1::3,5}$ is the first $3$ elements of $\mathbbm{d}_5$ (i.e. the column vectors $(3.65445,3.56891,4.3413)^\top$ and $ (2.63005,2.35319,2.63309)^\top$, respectively). This is a $2-$dimensional plane lying in $\bR^3$. There is no loss of generality in assuming that the learner will pick a hypothesis whose restriction to $\mu$ is in $\cV^\mu_M$, as hypothesis lying outside of $\cV_\cM$ would be incorrect for all choices of reward functions (thus, strictly worse).

We pick a $1-$dimensional feature mapping $\phi: \cS \mapsto \bR$ (i.e. $\Phi \in \bR^{5 \times 1}$). We choose $\Phi$ such that it is a linear combination of the last two columns of $\mathbbm{d}$, i.e. 
$$
\Phi = \alpha \mathbbm{d}_4 + \beta \mathbbm{d}_5,
$$
which means that $\Phi$ is a vector lying inside $\cV_\cM$.
Our particular choice of $\alpha$ and $\beta$ give
\begin{align} \label{eq:rand_phi}
(\phi_1,\phi_2,\phi_3)^\top = -0.5874 \mathbbm{d}_{1::3,4} + 0.9354 \mathbbm{d}_{1::3,5} = (0.313528,0.104797,-0.0870883)^\top
\end{align}
The only thing left to pick now is $\mu$. We cannot do this arbitrarily, as we have to ensure that $\norm{\Pi_\mu P}_\mu < \infty$. Following the characterization of Lemma \ref{lemma:pipmu}, we need to ensure that $\sum_j \mu_j \phi_j P_{j,4} =0$ and $\sum_j \mu_j \phi_j P_{j,5} = 0$. Since we have chosen $\phi$ and $P$, the above two equations are linear constraints in $\mu$. Together with the constraint that $\mu_1+\mu_2+\mu_3=1$, we can solve them to find that 
$$
(\mu_1,\mu_2,\mu_3)^\top = ( 0.0840949, 0.660425, 0.25548)^\top
$$

Note that such a solution---where $\mu$ is a valid distribution---is not always possible for different choices of $\bP$ and $\phi$, hence the seemingly mysterious choices for $\bP$ and $\phi$. This particular instance was found via a random search: we keep generating $P$ and the coefficients in Eq.\eqref{eq:rand_phi} for defining $\phi$, and stop when we find an instance with $\mu_1, \mu_2, \mu_3 > 0$ (they can be negative). 

It remains to show that 1) $\sigma_{\min}(\Sigma^{-1/2} A \Sigma^{-1/2}) = 0$, and 2) the worst-case asymptotic approximation error is $\infty$. For 1), one can verify that with this choice of $(P,\phi,\mu)$ we have:
$$
\Sigma = 0.0174572 > 0 \,\&\, A = 0 \implies \sigma_{\min}(\Sigma^{-1/2} A \Sigma^{-1/2}) = 0.
$$

The last thing to argue is that the error is $\infty$. The space of possible linear predictors is the line $\left\{\theta \cdot (\phi_1,\phi_2,\phi_3)\right\}$. Recall that we picked $\Phi$ such that this entire line lies inside $\cV_M$. In other words, there are an infinite number of possible realizable value functions that the environment could pick (obtained via $(r4,r5) = \theta (\alpha,\beta)$ for arbitrary $\theta$). 

However, from the perspective of the learner, the only information available is the value of the reward inside $\mu$ (which is $0$), the transitions inside $\mu$ (the matrix $\bP_{\mu,\mu}$), and the transitions from $\mu$ to $\neg \mu$ (the matrix $\bP_{\mu, \neg \mu}$). In fact we can assume that the learner knows the whole $\bP$ matrix, since $\bP_{\neg \mu, \mu} =0 $ and $\bP_{\neg \mu, \neg \mu} = \texttt{Id}_{2\times2}$ (self-loops). But none of this information is enough to deduce the value of $r_4,r_5$ (the reward happens at states that are unsupported), and this reward is what determines the true value function. So, for whatever value function the learner picks, we can pick a different realizable value function, thus rendering the approximation factor infinite. (See Figure \ref{fig:lstdgeometry}).

\begin{figure}[h]
\centering
\includegraphics[width=0.75\textwidth]{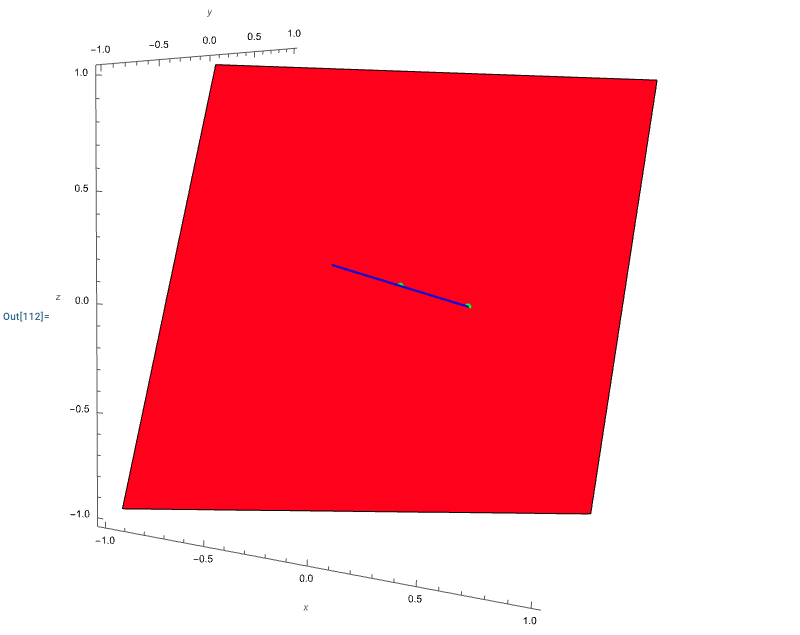}
\caption{The construction above. Red plane: $\cV_M$, space of value functions. Blue line: $\{\theta \cdot (\phi_1,\phi_2,\phi_3) \mid \theta \}$, space of linear predictors, which lies inside $\cV_M$. Two green points: a hypothesis value function and the other true value function.}\label{fig:lstdgeometry}
\end{figure}

\end{proof}

\subsection{Proof of Theorem \ref{thm:l2mulocal}}\label{app:l2mulocal}
This section proves the following result.

\Ltwomulocal*

\begin{proof}
The MRP has 5 states and $\gamma = 9/10$. Define $P$ which is rank-deficient but where $\bd \coloneqq (I - \gamma P)^{-1}$ is not (we find such an example via random search, see Appendix \ref{sec:P-defn}). We define $\mu$ later but for now we only need to know that it has support only on the first three states. We take the features 
\[
\Phi = \lambda_1 \bd_4 + \lambda_2 \bd_5 + \lambda_3 \psi
\]
where $d_4, d_5$, as in the proof of Theorem \ref{thm:sigmamin-lb}, are the 4th and 5th columns of $\bd$ respectively. The $\psi$ vector is a ``perturbation'' to be defined shortly (recall that if $\psi$ or $\lambda_3$ are $0$ then this gives $A=0$). We will take $\lambda_3$ to be small and $\lambda_1 = 1$. We will fix the 4th and 5th elements of $\psi$ to be 0, i.e. $\psi_4 = \psi_5 = 0$, and define $\psi_1,\psi_2,\psi_3$ later. Note that the features are $1$-dimensional, i.e. we have $\phi(s) \in \bR$ and $\Phi \in \bR^{5 \times 1}$. We define three instances $\cM_1, \cM_0, \cM_{-1}$, which are identical save for different reward functions indexed by $z \in \{-1,0,1\}$. Namely, the three possible reward functions are given by $r_4 = z \lambda_1 \cdot \eta$ and $r_5 = z \lambda_2 \cdot \eta$, where $\eta = 1/304$ is a normalization constant which we will see will ensure both that $|r_4|$ and $|r_5|$ are bounded by $1$ and that $\max_i \norm{\phi_i}_2 \leq 1$. For all three instances, the reward on supported states is fixed to be $r_1=r_2=r_3=0$. Using that $v_\cM = \bd \begin{pmatrix} r_1 \\ r_2 \\ r_3 \\ r_4 \\ r_5 \end{pmatrix} = r_4 d_4 + r_5 d_5$, we see that the misspecification error for the $\cM_1$ and $\cM_{-1}$ MRPs is
        \begin{equation}\label{eq:lambda-3-bound}
\norm{ \Pi_\mu v_\cM - v_\cM}_\mu = \inf_\theta\norm{(\theta \lambda_1 - r_4 )\bd_4 + (\theta \lambda_2 - r_5)\bd_5 + \theta \lambda_3 \psi}_\mu \leq |\lambda_3|  \norm{\psi}_\mu
\end{equation}
where in the inequality we have picked $\theta =  r_4 / \lambda_1 = r_5/\lambda_2 =  z \eta \in \{ \pm 1/304\}$ and used $|\theta| \leq 1$. Using that $d = I + \gamma P d$ we can check that $D(I - \gamma P) \Phi = \lambda_3 D(I- \gamma P)\psi$, and in particular
\[
A = \Phi^\top D (I - \gamma P) \Phi = \lambda_3 \Phi^\top D (I - \gamma P) \psi \implies \lambda_3 = \frac{\sigmamin{\Sigma^{-1}A}}{\Sigma^{-1}\Phi^\top D ( I - \gamma P) \psi } = \frac{\sigmamin{\Sigma^{-1/2}A\Sigma^{-1/2}}}{\Sigma^{-1}\Phi^\top D ( I - \gamma P) \psi }
\]
where in the above we have used at several points that the feature dimension is $1$.
Since the MRP corresponding to $z=0$ is realizable, the learner is forced to output $\theta = 0$ (else their approximation error is $\infty$). Let us calculate this estimator's approximation ratio on the other two MRPs:
\begin{align*}
    \alpha_\mu &\geq \frac{\norm{v_\cM}_\mu}{\norm{\Pi_\mu v_\cM - v_\cM}_\mu} \\
    &\geq  \frac{\norm{v_\cM}_\mu}{\lambda_3  \norm{\psi}_\mu} \tag{by Equation \eqref{eq:lambda-3-bound}}\\
    &= \frac{\norm{v_\cM}_\mu |\Sigma^{-1}\Phi^\top D(I - \gamma P)\psi|}{\sigmamin{\Sigma^{-1/2}A\Sigma^{-1/2}} \norm{\psi}_\mu} \tag{$\lambda_3$ equality from above}\\
     &= \frac{\norm{v_\cM}_\mu \norm{\Phi}_\mu |\Sigma^{-1}\Phi^\top D(I - \gamma P)\psi|}{\sigmamin{\Sigma^{-1/2}A\Sigma^{-1/2}} \norm{\Phi}_\mu \norm{\psi}_\mu} \\
    &= \frac{\norm{v_\cM}_\mu}{\sigmamin{\Sigma^{-1/2}A\Sigma^{-1/2}} \norm{\Phi}_\mu } \frac{\norm{\Phi\Sigma^{-1}\Phi^\top D(I - \gamma P)\psi}_\mu }{\norm{\psi}_\mu} \\
    &= \frac{\norm{v_\cM}_\mu}{\sigmamin{\Sigma^{-1/2}A\Sigma^{-1/2}} \norm{\Phi}_\mu } \frac{\norm{\Pi_\mu(I - \gamma P)\psi}_\mu}{\norm{\psi}_\mu}.
\end{align*}
Pick $\norm{\psi}_2 = 1$ such that $\norm{\Pi_\mu (I - \gamma P) \psi}_\mu = \norm{\Pi_\mu (I - \gamma P)}_\mu$. This exists by the reasoning of Appendix \ref{app:fixed-point}. By this property and $\norm{\psi}_\mu \leq \norm{\psi}_2$ we have: 
\[
\alpha_\mu \geq \frac{ \norm{v_\cM}_\mu }{\norm{\Phi}_\mu} \frac{\norm{\Pi_\mu(I - \gamma P)}_\mu}{\sigmamin{\Sigma^{-1/2}A\Sigma^{-1/2}}}.
\]
Note that $\frac{\norm{v_\cM}_\mu}{\norm{\Phi}_\mu}$ can be lower-bounded as follows:
\[
\frac{\norm{v_\cM}_\mu}{\norm{\Phi}_\mu} = \frac{\norm{\Phi - \lambda_3 \psi}_\mu}{\norm{\Phi}_\mu} \geq \frac{\norm{\Phi}_\mu - \lambda_3 \norm{\psi}_\mu}{\norm{\Phi}_\mu} = 1 - \lambda_3 \frac{\norm{\psi}_\mu}{\norm{\Phi}_\mu} =  1 - \frac{\sigmamin{\Sigma^{-1/2}A\Sigma^{-1/2}}\norm{\psi}_\mu}{\norm{\Pi_\mu(I-\gamma P)\psi}_\mu} = 1 - \frac{\sigmamin{\Sigma^{-1/2}A\Sigma^{-1/2}}}{\norm{\Pi_\mu(I-\gamma P)}_\mu} .
\]
Plugging this in, we get:
\[
 \alpha_\mu \geq \left(1 - \frac{\sigmamin{\Sigma^{-1/2}A\Sigma^{-1/2}}}{\norm{\Pi_\mu(I-\gamma P)}_\mu}\right) \frac{\norm{\Pi_\mu(I - \gamma P)}_\mu}{\sigmamin{\Sigma^{-1/2}A\Sigma^{-1/2}}} = \frac{\norm{\Pi_\mu(I - \gamma P)}_\mu}{\sigmamin{\Sigma^{-1/2}A\Sigma^{-1/2}}} - 1
 \]
The only thing left to show is to pick the values of $(\lambda_1,\lambda_2,\lambda_3)$ and $\mu$ to show that we can set the value of $\frac{\norm{\Pi_\mu P}_\mu}{\sigmamin{\Sigma^{-1/2}A\Sigma^{-1/2}}} = x$. This follows by the reasoning of Appendix \ref{sec:mu-defn}.

\end{proof}

\subsubsection{Defining $P$}\label{sec:P-defn}
We denote $P = \begin{pmatrix}
| & | & | & | & | \\
p_{1} & p_{2} & p_{3} & p_{4} & p_{5} \\
| & | & | & | & | \\
\end{pmatrix}$ and likewise $\bd = \begin{pmatrix}
| & | & | & | & | \\
d_{1} & d_{2} & d_{3} & d_{4} & d_{5} \\
| & | & | & | & | \\
\end{pmatrix}$ where $\bd = (I - \gamma P)^{-1}$. We also call $P_{\neg \mu} = \begin{pmatrix}
| & | \\
 p_{4} & p_{5} \\
 | & | \\
\end{pmatrix}$ and likewise for $d_{\neg \mu}$. Recall that $\gamma = 9/10$. We take any $P$ matrix such that $\rank(P_{\neg \mu}) = 1$, $\rank(d_{\neg \mu}) = 2$, and $P_{1,2}=P_{1,3} = P_{2,3} = P_{3,2} = 0$. A random search finds the following example:

$$
\bP = \begin{pmatrix}
0.384931 & 0 & 0 & 0.393873 & 0.221196\\
0.0864944 & 0.784211 & 0 & 0.0827968 & 0.046498\\
0.575606 & 0.35247 & 0 & 0.0460586 & 0.0258661\\
0.346009 & 0.227495 & 0.00896672 & 0.267374 & 0.150155\\
0.492524 & 0.0488124 & 0.364725 & 0.0601558 & 0.033783\\
\end{pmatrix},
$$
and of course the discounted occupancy matrix is 
$$
\mathbbm{d} \coloneqq (I - \gamma \bP)^{-1} = \begin{pmatrix}
3.9985 & 2.16139 & 0.421258 & 2.18933 & 1.22951\\
2.21064 & 4.98011 & 0.308174 & 1.60162 & 0.899458\\
2.96132 & 2.86469 & 1.34819 & 1.80957 & 1.01624\\
2.81682 & 2.67477 & 0.392056 & 2.99562 & 1.12073\\
3.08799 & 2.33295 & 0.684865 & 1.85336 & 2.04083\\
\end{pmatrix}.
$$

The condition $P_{1,2}=P_{1,3} = P_{2,3} = P_{3,2} = 0$ is chosen to enable more control over the magnitude of the parameter $x = \norm{\Pi_\mu P}_\mu / \sigmamin{\Sigma^{-1/2}A\Sigma^{-1/2}}$. The two columns of $P_{\neg \mu}$ are scalar multiples of each other, which verifies that $P_{\neg \mu}$ is rank 1. Similarly, it is easy to see that $d_{\neg \mu}$ is rank 2 since the columns cannot be expressed as scalar multiples of each other.

\subsubsection{Defining $\psi$ and $\Phi$}\label{app:fixed-point}

We start with a lemma.
\begin{lemma}\label{lemma:rank-deficient}
Let $P$, $\mu$, $\psi$ be arbitrary, and $\Phi = \lambda_1 d_4 + \lambda _2 d_5 + \lambda_3 \psi$ for some $(\lambda_1,\lambda_2,\lambda_3)$.  Assume that $P$ is such that $P_{\neg \mu}$ is rank-deficient (with rank $1$). Then there exists a vector $(\lambda_1, \lambda_2,\lambda_3)$ such that $\norm{\Pi_\mu P}_\mu < \infty$. Furthermore we can find such a vector such that $\lambda_3 \neq 0$, and one of $\lambda_1$ or $\lambda_2$ do not equal $0$. 
\end{lemma}
\begin{proof}
We want $\vec{\lambda}$ to be such that the constraints of Lemma \ref{lemma:pipmu} hold, so that $\norm{\Pi_\mu P}_\mu < \infty$. We now show that this is equivalent to rank-deficiency of some matrix. The rank-deficiency of this new matrix will hold by the presupposition that $P_{\neg \mu}$ is rank-deficient. Plugging in the definition $\Phi$ and expanding the equation in Lemma \ref{lemma:pipmu} gives us that we need $\lambda_1$, $\lambda_2$, $\lambda_3$ such that:
\begin{align*}
&\lambda_1\left(\mu_1 d_{1,4}P_{1,4} + \mu_2 d_{2,4}P_{2,4} + \mu_3 d_{3,4}P_{3,4}\right) + \lambda_2 \left(\mu_1 d_{1,5}P_{1,4} + \mu_2 d_{2,5}P_{2,4} + \mu_3 d_{3,5}P_{3,4} \right)  + \lambda_3\left(\mu_1 \psi P_{1,4} + \mu_2 \psi P_{2,4} + \mu_3 \psi P_{3,4}\right) = 0 \\
&\lambda_1\left(\mu_1 d_{1,4}P_{1,5} + \mu_2 d_{2,4}P_{2,5} + \mu_3 d_{3,4}P_{3,5}\right) + \lambda_2 \left(\mu_1 d_{1,5}P_{1,5} + \mu_2 d_{2,5}P_{2,5} + \mu_3 d_{3,5}P_{3,5} \right) + \lambda_3\left(\mu_1 \psi P_{1,5} + \mu_2 \psi P_{2,5} + \mu_3 \psi P_{3,5}\right) = 0 
\end{align*}
which is equivalent to 
\begin{align*}
&\lambda_1 \langle d_{4}, p_{4} \rangle_\mu + \lambda_2 \langle d_{5}, p_{4} \rangle_\mu  + \lambda_3 \langle \psi, p_{4} \rangle= 0 \\
&\lambda_1 \langle d_{4},p_{5} \rangle_\mu + \lambda_2 \langle d_{4}, p_{5} \rangle_\mu + \lambda_3 \langle \psi, p_{5} \rangle = 0,
\end{align*}
or 
\begin{equation}\label{eq:M-matrix-kernel}
\begin{pmatrix}
\langle D d_4, p_{:,4} \rangle &  \langle D d_5, p_{:,4} \rangle & \langle D \psi, p_{:,4} \rangle \\
\langle D d_4, p_{:,5} \rangle & \langle D d_4, p_{:,5} \rangle & \langle D \psi, p_{:,5} \rangle
\end{pmatrix} \begin{pmatrix} \lambda_1 \\ \lambda_2 \\ \lambda_3 \end{pmatrix} = \begin{pmatrix} 0 \\ 0\end{pmatrix}\implies  \underbrace{\begin{pmatrix}
- p_{:,4} - \\
- p_{:,5} - \\
\end{pmatrix} D 
\begin{pmatrix}
| & | & | \\
d_4 & d_5 & \psi \\
| & | & | \\
\end{pmatrix}}_{\coloneqq M(\psi)}\begin{pmatrix} \lambda_1 \\ \lambda_2 \\ \lambda_3 \end{pmatrix} = \begin{pmatrix} 0 \\ 0\end{pmatrix}.
\end{equation} 

Note that a non-zero solution to this system of equations exists whenever $M(\psi)$ is rank-deficient, and furthermore that $\rank(M(\psi)) \leq \rank(P_{\neg \mu})$ since $P_{\neg \mu}$ is the first matrix in the definition of $M(\psi)$. Furthermore note that $M(\psi)$ cannot be trivial since that would imply that $p_i^\top D d_j$ for $i,j \in \{4,5\}$, which is not possible as the vectors are all non-negative and have at least one non-zero element. Since we have $\rank(P_{\neg \mu}) = 1$ by construction, then we have $\rank(M(\psi))=1$, in which case the solution space of the above equation (i.e., the kernel of $M(\psi)$) is $2$-dimensional by the rank-nullity theorem. Thus there exists a non-zero vector that satisfies Lemma \ref{lemma:pipmu} and thus $(\Pi_\mu P)_{\mu, \neg \mu} = 0$.

We show that furthermore we can find a vector in the kernel such that $\lambda_3 \neq 0$. Recall that the kernel is 2-dimensional and is a subspace of $\bR^3$, so the only case to rule out is that the kernel is exactly the plane $\{(x,y,0) \mid x, y \in \bR\}$ (i.e. the $x-y$ plane). We can rule out this case by examining the sub-matrix $M' = \begin{pmatrix}
- p_{:,4} - \\
- p_{:,5} - \\
\end{pmatrix} D 
\begin{pmatrix}
| & | \\
d_4 & d_5 \\
| & |  \\
\end{pmatrix}$. If the kernel was the $x-y$ plane then this matrix would also have a $2$-dimensional kernel. However this is not possible since, as we shall show next, the rank of this matrix is exactly $1$. First note that the rank of $M'$ is also at most $1$ (due to rank-deficiency of $P_{\neg \mu}$, which still appears as the first matrix in $M'$), and furthermore that it has to be exactly one since the matrix cannot be trivial (all of its elements are positive). So the kernel of $M'$ is also exactly $1$ by rank-nullity. Thus we can find $\lambda_3 \neq 0$ in the kernel. The fact that one of $\lambda_1,\lambda_2 \neq 0$ is because the kernel is 2-dimensional, and thus it must have a dimension along which $\lambda_1$ or $\lambda_2 \neq 0$ (i.e. it cannot be the line $\{(0,0,t \cdot \lambda_3)\}_t$).

\end{proof}

Let $\Phi(\psi, \lambda) = \lambda_1 d_4 + \lambda_2 d_5 + \lambda_3 \psi$. We define a function $h: \psi \mapsto \lambda$ that maps every $\psi$ to a choice of $\lambda$ satisfying Lemma \ref{lemma:rank-deficient}. Write $\Phi(\psi)$ for the function $\Phi(\psi,h(\psi))$.
\begin{lemma}\label{lemma:phi-psi-continuous}
 For all $c > 0$, there exists a continuous function $h_c$ that maps every $\psi$ onto $\lambda$ satisfying Lemma \ref{lemma:rank-deficient}, and furthermore satisfies  $\lambda_1 = 1$ and $\lambda_3 = c$. Furthermore, the function
    $\Phi(\psi)$ is continuous as a function of $\psi$, i.e. as $\psi_n \rightarrow \psi$ we have $\Phi(\psi_n, h(\psi_n)) \rightarrow \Phi(\psi,h(\psi))$.
    
\end{lemma}
\begin{proof}
Note that the conditions of Lemma \ref{lemma:rank-deficient} are satisfied whenever $\lambda$ lies inside the kernel of $M(\psi)$ and $\lambda_3 \neq 0$. For $c > 0$, we construct some function $f_c(\psi,\lambda)$ who's zeros corresponds to these conditions for $\lambda$ being satisfied, and explicitly define the function $h_c(\psi)$ along this level set. The function is $f_c(\psi, \lambda) = M(\psi) \lambda + \lambda_3 \begin{pmatrix} 0 \\ 0 \\ c \end{pmatrix}$,  where $M(\psi) = \begin{pmatrix}
- p_{:,4} - \\
- p_{:,5} - \\
\end{pmatrix} D 
\begin{pmatrix}
| & | & | \\
d_4 & d_5 & \psi \\
| & | & | \\
\end{pmatrix}$ is the matrix from Lemma \ref{lemma:rank-deficient}. Note that $f_c: \bR^{5 + 3} \rightarrow \bR^{3}$ and that $f_c(\psi,\lambda) = 0 \iff \left(M(\psi) \lambda = \begin{pmatrix} 0 \\ 0 \end{pmatrix} \text{ and } \lambda_3 = c\right)$. 
Let $\lambda$ be a solution to $f_c(\psi,\lambda) = 0$ (which exists since we can find $\lambda$ in the kernel for which $\lambda_3 \neq 0$, thus re-scaling can achieve $\lambda_3 = c$). Since $\lambda_3 =c$, then we must have 
\[
    \bar{M} \begin{pmatrix}
        \lambda_1 \\ \lambda_2
    \end{pmatrix} = - c M_3(\psi)
\]
where $\bar{M} = \begin{pmatrix} | & | \\ M_1(\psi) &M_2(\psi) \\ | & | \end{pmatrix}$ is the submatrix corresponds to the first two columns of $M(\psi)$ and $M_3(\psi)$ is the third column of $M(\psi)$. Let $\bar{\lambda} = \begin{pmatrix} \lambda_1 \\ \lambda_2 \end{pmatrix}$. Note that $\bar{M}$ is independent of $\psi$. Again, the rank of $\bar{M}$ is 1 so the solution space to this equation is 1-dimensional. %
Since $\bar{M}$ is rank $1$, it is equivalent to $ab^\top$ for two vectors $a,b$. Without loss of generality we can take $a = \begin{pmatrix} 1 \\ \zeta\end{pmatrix}$ and $b = \begin{pmatrix} \langle p_4, d_4\rangle_\mu \\ \langle p_4, d_5\rangle_\mu \end{pmatrix}$, where $ 1 \leq \zeta \leq 2$ is the scalar multiple s.t. $p_4 = \zeta p_5$ (which exists since $P_{\neg \mu}$ is rank 1). This gives
\[
 (a b^\top) \bar{\lambda} = a (b^\top \bar{\lambda}) = - c M_3(\psi).
\]
Taking norms, this gives us the relationship 
\[
c \frac{\norm{M_3(\psi)}}{\norm{a}} = | b^\top \bar{\lambda} |
\]
or equivalently
\[
    b^\top \bar{\lambda} = c \cdot \sign(\langle a, M_3(\psi) \rangle) \frac{\norm{M_3(\psi)}}{\norm{a}}.
\]
We can take the solution which is $\lambda_1 = 1$ and 
\[
\lambda_2 = \left(c \cdot \sign(\langle a, M_3(\psi)\rangle) \frac{\norm{M_3(\psi)}}{\norm{a}} - b_1\right)/b_2.
\]

The choice of $h(\psi) = (\lambda_1, \lambda_2, \lambda_3) = (1, \frac{c \cdot \sign(\langle a, M_3(\psi)\rangle) \frac{\norm{M_3(\psi)}}{\norm{a}} - b_1}{b_2}, c)$ is continuous as a function of $\psi$, since $\psi$ only appears via $\norm{M_3(\psi)}$, and therefore $\Phi(\psi, \lambda) = M(\psi) h(\psi)$ is continuous. 
\end{proof}

\paragraph{Normalizing the features} 
While we are here, let us check the norm bound of $\lambda_1,\lambda_2,\lambda_3$ and $\norm{\phi_i}$. We have claimed that $\eta = 1/304$ is a sufficient normalization constant to ensure that $|\lambda_i| \eta \leq 1$ and $\norm{\phi_i}_2 \eta \leq 1$. 

Evidently we have $|\lambda_1| \leq 1$ and $|\lambda_3| \leq 1$ if $c \leq 1$ ($c$ has not been chosen yet but will be shown to satisfy this). For $\lambda_2$, we have
\[
|\lambda_2| \leq \left( c \frac{\norm{M_3(\psi)}}{\norm{a}} + |b_1| \right) / \min_{i \in {1,2,3}} p_4^i d_5^i.
\]
To bound this, we can use the bounds that $\norm{a} = \sqrt{1 + \zeta^2} \geq \sqrt{2}$, $|b_1| = \langle p_4,d_4\rangle_\mu \leq \langle p_4,d_4 \rangle = \norm{p_4}\norm{d_4} \leq 2$ and $\min_i p_4^i d_5^i \geq 0.08$. We also have that $M_3(\psi) = \begin{pmatrix} \langle p_4, \psi \rangle_\mu \\ \langle p_4, \psi \rangle_\mu \end{pmatrix}$ so we have 
\[
\norm{M_3(\psi)}^2 \leq \langle p_4,\psi \rangle ^2 + \langle p_5, \psi \rangle^2 \leq \norm{p_4}^2 + \norm{p_5}^2 \leq 50.
\]
Combining everything gives us
\[
|\lambda_2| \leq \left( \sqrt{25} + 2 \right) / 0.08 \leq 100.
\]

Let us also check the feature magnitude $\norm{\phi_i}$ so that we can normalize the features. We have:
\[
    \norm{\phi_i} = |\phi_i| \leq |\lambda_1| |d_4^i| + |\lambda_2| |d_5^i| + |\lambda_3| |\psi_i| \leq 1\cdot 3 + 100 \cdot 3 + 1 \cdot 1 = 304.
\]
So let us rescale subsequent and prior choices of $(\lambda_1,\lambda_2,\lambda_3)$ by $1/304$ so that we are ensured to have $\norm{\phi_i} \leq 1$ for all $i$.

We will need the following lemma about continuity of eigenvectors.
\begin{lemma}[\cite{ortega1990numerical} Lemma 3.1.3]\label{lem:cont-eigenvector}
    Let $\lambda$ be a simple eigenvalue of some matrix $A \in \bR^{n \times n}$, with associated eigenvector $x \neq 0$. Then, for $E \in \bR^{n \times n}$, $A+E$ has eigenvalue $\lambda(E)$ and $x(E)$ where 
    \[
    \lambda(E) \rightarrow \lambda \,\text{ and }\, x(E) \rightarrow x \text{ as } E \rightarrow 0.
    \]
\end{lemma}

The main result of this section is:
\begin{lemma}\label{lemma:psi-fixed-point}
Let $P$ be such that $P_{\neg \mu}$ is rank-deficient, and $\Phi_c(\psi) \coloneqq \Phi(\psi, h_c(\psi))$ be as above. %
Denote $\Pi_\mu^{\psi}$ the projection matrix for features $\Phi_c(\psi)$. Then, for all $\gamma$,  and $\mu$ such that $\supp(\mu) = \{1,2,3\}$, there exists a $\psi$ such that $\psi_4=\psi_5=0$, $\norm{\psi}_2=1$, and
\[
\norm{\Pi_\mu^{\psi}(I - \gamma P) \psi}_\mu = \norm{ \Pi_\mu^{\psi}(I - \gamma P)}_\mu,
\]
i.e. $\psi$ realizes the $L_2(\mu)$ operator norm of $\Pi^\psi_\mu (I - \gamma P)$.
\end{lemma}
\begin{proof}
Fix $\psi$ s.t. $\psi_4=\psi_5=0$ and an associated $h_c(\psi)$ given by Lemma \ref{lemma:phi-psi-continuous}.  Note that 
\begin{align*}
\norm{\Pi^{\psi}_\mu (I - \gamma P)}_\mu 	&= \max_{\norm{x}_\mu=1} \norm{\Pi^{\psi}_\mu (I - \gamma P) x}_\mu \\
&= \max_{\norm{x}_\mu=1} \norm{D^{1/2}\Pi_\mu^{\psi}(I - \gamma P) x}_2 \\
&= \max_{\norm{D^{1/2} x}_2 = 1} \norm{D^{1/2}\Pi_\mu^{\psi}(I - \gamma P) x}_2 \\
&= \max_{\norm{y}_2 = 1} \norm{D^{1/2}\Pi_\mu^{\psi}(I - \gamma P) D^{-\dagger/2} y}_2 \\
&= \norm{\underbrace{D^{1/2}\Pi_\mu^{\psi}(I - \gamma P) D^{-\dagger/2}}_{\coloneqq N(\psi)}}_2.
\end{align*}

Here the second-to-last line follows since $(\Pi_\mu^{\psi} P)_{\mu, \neg \mu} = 0 \iff (\Pi_\mu^{\psi} (I - \gamma P))_{\mu, \neg \mu} = 0$, and thus we can assume that $x_{\neg \mu} = 0$ and use the substitution $D^{1/2} x = y \iff x = D^{-\dagger/2} y$. 

The $L_2$ operator norm at the end of the last display has as value $\sigma_{\max}(N(\psi)) = \sqrt{\lambda_{\max}((N(\psi))^\top N(\psi))}$. We need to know which vector maximizes this operator norm, and since there are multiple such vectors we make the canonical choice that it be the (normalized) eigenvector corresponding to this max eigenvalue, denoted $v_{\max}$. 

Let $\bar{\psi}$ denote the first $3$ elements of $\psi$, and recall that $\psi = \bar{\psi} \oplus (0,0)$ since $\psi_4=\psi_5=0$. We can define the function $g : \{\bar{\psi} \in \bR^3: \norm{\bar{\psi}}_2 = 1\} \mapsto \{\bar{\psi} \in \bR^3: \norm{\bar{\psi}}_2 = 1\}$ s.t.
\[
g(\bar{\psi} \oplus (0,0)) = g(\psi) = v_{\max}((N(\psi))^\top N(\psi)),
\]
and to conclude the proof we have to show that this function has a fixed point $\bar{\psi}$. Since $g$ maps to and from the unit ball in $\bR^3$, by Brouwer's fixed point theorem it only remains to show that this function is continuous as a function of $\bar{\psi}$. 

Continuity will follow from the continuity of eigenvectors given by Lemma \ref{lem:cont-eigenvector}. Note that the Lemma applies, since the eigenvalue $\lambda_{\max}$ is indeed simple as $N(\psi)$ is rank 1, and thus has only one non-zero eigenvalue (and in particular, $\lambda_{\max}$ appears with multiplicity one). To apply this lemma, we re-write $N(\psi_n)$ in the form of $N(\psi)+E$ given in that lemma, where $E$ is a matrix which goes to $0$. The only term in $N(\psi_n) = D^{1/2} \Pi_\mu^{\psi_n} (I - \gamma P) D^{-\dagger/2}$ which changes as $\psi_n \rightarrow \psi$ is $\Pi_\mu^{\psi_n}$. Let $\psi_n \rightarrow \psi$, and note that by continuity of $\Phi(\psi)$ (Lemma \ref{lemma:phi-psi-continuous}) we can write $\Phi(\psi) = \Phi(\psi_n) + \varepsilon_n$ where $\varepsilon_n \rightarrow 0$ as $\psi_n \rightarrow \psi$. Then, we have that:
\begin{align}
\Pi_\mu^{\psi_n} &= \Phi(\psi_n) \Sigma^{-1}_{\psi_n} \Phi(\psi_n)^\top D \nonumber \\
&= (\Phi(\psi) + \varepsilon_n)\Sigma^{-1}_{\psi_n} (\Phi(\psi) + \varepsilon_n)^\top D \nonumber \\
&= \Phi(\psi)\Sigma^{-1}_{\psi_n} \Phi(\psi)^\top D + \varepsilon_n\Sigma^{-1}_{\psi_n}\Phi(\psi)^\top D + \Phi(\psi)\Sigma^{-1}_{\psi_n}\varepsilon_n^\top D+ \varepsilon_n\Sigma^{-1}_{\psi_n} \varepsilon_n^\top D \label{eq:gobble}
\end{align}
Note firstly that $\Sigma^{-1}_\psi$ and $\Sigma^{-1}_{\psi_n}$ are bounded since $\mu$ has full support over the first three states and $\psi, \psi_n$ need to be non-zero on one of their first three components by the condition $\norm{\bar{\psi}}_2 = 1$. Thus, for fixed $\mu$, as $ n \rightarrow 0$, the 2nd, 3rd, and 4th terms vanish due to the fact that $\varepsilon_n \rightarrow 0$. It remains to deal with the first term ($\Phi(\psi)\Sigma^{-1}_{\psi_n} \Phi(\psi)^\top D$). Note that $\Sigma_{\psi_n}$ is a scalar since $d=1$, so $\Sigma_{\psi_n}^{-1}= \frac{1}{\Sigma_{\psi} + \delta_n}$, where $\delta_n = \varepsilon_n^\top D \Phi(\psi) + \Phi(\psi)^\top D\varepsilon_n + \varepsilon_n^\top D \varepsilon_n$. For small enough $\varepsilon_n$ (such that $\delta_n/\Sigma_\psi < 1$) this can be written as a Taylor series $\Sigma_{\psi_n}^{-1}= \frac{1}{\Sigma_{\psi} + \delta_n} = \frac{1}{\Sigma_{\psi}} - \frac{\delta_n}{\Sigma^2_{\psi}} + \frac{\delta_n^2}{\Sigma^3_{\psi}} - ...$. 
Plugging this into Equation \eqref{eq:gobble} gives that 
$\Pi_\mu^{\psi_n} = \Pi_\mu^{\psi} + E$, where $E = \Phi(\psi)(-\frac{\delta_n}{\Sigma^2_{\psi}} + \frac{\delta_n^2}{\Sigma^3_{\psi}} + ...)\Phi(\psi)^\top D + \varepsilon_n\Sigma^{-1}_{\psi_n}\Phi(\psi)^\top + \Phi(\psi)\Sigma^{-1}_{\psi_n}\varepsilon_n^\top + \varepsilon_n\Sigma^{-1}_{\psi_n} \varepsilon_n^\top D$. Note that $E \rightarrow 0$ as $\varepsilon_n \rightarrow 0$ (or equivalently as $\psi_n \rightarrow \psi$).
Thus we can apply Lemma \ref{lem:cont-eigenvector} to obtain that $g(\psi)$ is indeed continuous.

\end{proof}

To conclude this section, we note that the value of $c$ given in $h_c(\psi)$ (Lemma \ref{lemma:phi-psi-continuous}) has not been chosen yet. This will be chosen in the following section. 

\subsubsection{Defining $\mu$ and $\lambda$}\label{sec:mu-defn}

We pick $\mu$ and $\lambda_3$ so that we can satisfy $x = \frac{\norm{\Pi_\mu P}_\mu}{ \sigmamin{\Sigma^{-1/2} A \Sigma^{-1/2}} }$. Towards this, let's simplify the expression $\norm{\Pi_\mu P}_\mu$. Recall that for each $\psi$ we pick $h(\psi)$ s.t. Lemma \ref{lemma:rank-deficient} holds, and in particular $(\Pi_\mu P)_{\mu, \neg \mu} = 0$. Note that in the $1-$dimensional case, for any $v \in \bR^{S}$, we have 
\begin{align*}
\norm{\Pi_\mu P v}^2_\mu &= \mu_1 \left( \langle \phi_1, \Sigma^{-1} (\sum_k \mu_k \phi_k P_k^\top v) \rangle \right)^2 + \mu_2 \left( \langle \phi_2, \Sigma^{-1} (\sum_k \mu_k \phi_k P_k^\top v) \rangle\right)^2 + \mu_3 \left( \langle \phi_3, \Sigma^{-1} (\sum_k \mu_k \phi_k P_k^\top v) \rangle\right)^2 \\
&= (\Sigma^{-1})^2 \left(\mu_1 \phi_1^2 (\sum_k \mu_k \phi_k P_k^\top v)^2 + \mu_2 \phi_2^2 (\sum_k \mu_k \phi_k P_k^\top v)^2  + \mu_3 \phi_3^2 (\sum_k \mu_k \phi_k P_k^\top v)^2\right)  \tag{because $d=1$}\\
&= (\Sigma^{-1})^2(\mu_1 \phi_1^2 + \mu_2 \phi_2^2 + \mu_3 \phi_3^3) (\sum_k \mu_k \phi_k P_k^\top v)^2 \\
&= (\Sigma^{-1})^2 (\norm{\Phi}^2_\mu) (\sum_k \mu_k \phi_k P_k^\top v)^2\\
&= \frac{1}{\norm{\Phi}^2_\mu} (\sum_k \mu_k \phi_k P_k^\top v)^2.   \tag{when $d=1$, $\Sigma = \Phi^\top D \Phi = \norm{\Phi}^2_\mu$}
\end{align*} 
Taking the square roots gives that 
\[
\norm{\Pi_\mu P v}_\mu = \frac{1}{\norm{\Phi}_\mu} | \sum_k \mu_k \phi_k P_k^\top v |.
\]

Note that the summation can be further broken down into 
\begin{equation*}
\sum_k \mu_k \phi_k (P^\mu_k)^\top v_\mu + \sum_k \mu_k \phi_k (P^{\neg \mu}_k)^\top v_{\neg \mu},
\end{equation*}
where $v = v_\mu \oplus v_{\neg \mu} = (v_1,v_2,v_3) \oplus (v_4,v_5)$ and similarly $P_k = P^{\mu}_k \oplus P^{\neg \mu}_k$ (i.e. $P^{\mu}_k = (P_{k,1},P_{k,2},P_{k,3})$ and $P^{\neg \mu}_k=(P_{k,4},P_{k,5})$). The second term in this summation is $0$ for all $v$ by Lemma \ref{lemma:pipmu}. Thus, $\norm{\Pi_\mu P}_\mu = \frac{1}{\norm{\Phi}_\mu} \max_{\norm{v}_\mu =1} | \sum_k \mu_k \phi_k (P_k^\mu)^\top v_\mu|$. Let us call the vectors $\nu_i \coloneqq \phi_i (P_i^\mu) \in \bR^3$, and we treat them as fixed. Note further that %
\begin{align*}
\norm{\Pi_\mu P}_\mu &= \frac{1}{\norm{\Phi}_\mu} \max_{\norm{v}_\mu =1} | \sum_k \mu_k \nu_k^\top v_\mu| \\
&= \frac{1}{\norm{\Phi}_\mu} \max_{\norm{v}_\mu =1} | \left(\sum_k \mu_k \nu_k\right)^\top v_\mu| \\
&= \frac{1}{\norm{\Phi}_\mu} \max_{\norm{D^{1/2}_\mu v_\mu}_2 =1} | \left(\sum_k \mu_k\nu_k\right)^\top v_\mu|,
\end{align*}
where in the last line we are using the $2-$norm over $\bR^3$, and $D^{1/2}_\mu \coloneqq D^{1/2}\vert_\mu$ is the restriction of $D^{1/2}$ to the first 3 coordinates. Note that $D^{-1/2}_\mu \coloneqq (D^{1/2}_\mu)^{-1}$ is well-defined. Hence, 
\begin{align}
\norm{\Pi_\mu P}_\mu &= \frac{1}{\norm{\Phi}_\mu} \max_{\norm{y}_2 =1} | \left(\sum_k \mu_k \nu_k \right)^\top D^{-1/2}_\mu y| \tag{change of coordinates $y = D^{-1/2}_\mu v_\mu$} \nonumber\\
&= \frac{1}{\norm{\Phi}_\mu} \max_{\norm{y}_2 =1} | y^\top \left( D^{-1/2}_\mu \sum_k \mu_k \nu_k\right)| \nonumber \\
&= \frac{1}{\norm{\Phi}_\mu} \norm{D^{-1/2}_\mu\left(\sum_k \mu_k \nu_k\right)}_2 \tag{max is achieved by $y = \frac{D^{-1/2}_\mu \sum_k \mu_k \nu_k}{ \norm{D^{-1/2}_\mu \sum_k \mu_k \nu_k}_2}$} \nonumber \\
&= \frac{1}{\norm{\Phi}_\mu} \norm{D^{-1/2}_\mu\left(\mu_1 \nu_1 + \mu_2 \nu_2 + \mu_3 \nu_3 \right)}_2  \nonumber \\
&= \frac{1}{\norm{\Phi}_\mu} \norm{\begin{pmatrix}
\sqrt{\mu_1} \nu_{1,1}\\
\frac{\mu_1}{\sqrt{\mu_2}} \nu_{1,2}\\
\frac{\mu_1}{\sqrt{\mu_3}} \nu_{1,3}
\end{pmatrix}  +\begin{pmatrix}
 \frac{\mu_2}{\sqrt{\mu_1}}\nu_{2,1}\\
\sqrt{\mu_2} \nu_{2,2}\\
 \frac{\mu_2}{\sqrt{\mu_3}} \nu_{2,3}
\end{pmatrix} + \begin{pmatrix}
 \frac{\mu_3}{\sqrt{\mu_1}}\nu_{3,1}\\
 \frac{\mu_3}{\sqrt{\mu_2}}\nu_{3,2}\\
 \sqrt{\mu_3} \nu_{3,3}
\end{pmatrix}}_2  \nonumber \\
&= \frac{1}{\norm{\Phi}_\mu} \norm{\phi_1\begin{pmatrix}
\sqrt{\mu_1} P_{1,1}\\
\frac{\mu_1}{\sqrt{\mu_2}} P_{1,2}\\
\frac{\mu_1}{\sqrt{\mu_3}} P_{1,3}
\end{pmatrix}  + \phi_2 \begin{pmatrix}
 \frac{\mu_2}{\sqrt{\mu_1}}P_{2,1}\\
\sqrt{\mu_2} P_{2,2}\\
 \frac{\mu_2}{\sqrt{\mu_3}} P_{2,3}
\end{pmatrix} + \phi_3\begin{pmatrix}
 \frac{\mu_3}{\sqrt{\mu_1}}P_{3,1}\\
 \frac{\mu_3}{\sqrt{\mu_2}}P_{3,2}\\
 \sqrt{\mu_3} P_{3,3}
\end{pmatrix}}_2. \label{eq:pi-mu-p-equals-x}\\
\end{align}
Plugging in that $P$ is 0 on certain indices, the quantity we care about is 
\[
  x = \frac{\norm{\Pi_\mu P}_\mu}{\Sigma^{-1}|A|} =  \frac{\Sigma^{1/2}}{|A|} \norm{\phi_1\begin{pmatrix}
\sqrt{\mu_1} P_{1,1}\\
    0\\
     0
\end{pmatrix}  + \phi_2 \begin{pmatrix}
 0\\
\sqrt{\mu_2} P_{2,2}\\
  0
\end{pmatrix} + \phi_3\begin{pmatrix}
 \frac{\mu_3}{\sqrt{\mu_1}}P_{3,1}\\
 \frac{\mu_3}{\sqrt{\mu_2}}P_{3,2}\\
 \sqrt{\mu_3} P_{3,3}
\end{pmatrix}}_2.
\]

We claim that this function is continuous along the path $(\mu_1, \frac{1- \mu_1}{2}, \frac{1-\mu_1}{2})$ in the domain $\mu_1 \in [0,1]$. Call this function $\rho(\mu_1)$. To see continuity, note that $\Sigma^{1/2}$ and $A$ are linear function of $\mu$ (since they are expectations under $\mu$, and can thus be written as an $L_2$ inner product where one of the vectors is $\mu$, e.g. $A = \left\langle \mu, \phi(\cdot)\left(\phi(\cdot) - \gamma \bE_{s' \sim P(\cdot)}[\phi(s')] \right) \right\rangle$. Similarly, the $L_2$ norm can be seen to be continuous as a function of $\mu_1$, since expanding the $L_2$ norm shows that it is a sum of rational functions in $\mu_1$. Since $\rho$ is continuous as a function of $\mu_1$, it takes on all values between $\rho(0)$ and $\rho(1)$ by the intermediate value theorem. Let us check the range of this function. Let us write $\phi_i' = \bE_{s' \sim P(i)} \phi(s')$. Taking $\mu_1 \rightarrow 0$ causes the ratio $\Sigma^{1/2}/|A|$ to take the value
\begin{align*}
\frac{\sqrt{0.5}\left(\phi_2^2 + \phi_3^2\right)^{1/2}}{0.5 \left| \phi_2(\phi_2 - \gamma \phi'_2) + \phi_3(\phi_3 - \gamma \phi'_3) \right|} &= \frac{\sqrt{0.5}\left(\phi_2^2 + \phi_3^2\right)^{1/2}}{0.5 \left| \phi^2_2 -   \gamma \phi_2 \phi'_2 + \phi^2_3 - \gamma \phi_3\phi'_3 \right|} \\
&\geq \frac{\sqrt{0.5}\left(\phi_2^2 + \phi_3^2\right)^{1/2}}{0.5 \left| \phi^2_2 + \phi_3^2\right| +   \gamma \left| \phi_2 \phi'_2 + \phi_3\phi'_3 \right| } \\
& \geq \frac{\sqrt{0.5}\left(\phi_2^2 + \phi_3^2\right)^{1/2}}{0.5 \left| \phi^2_2 + \phi_3^2\right| +   \gamma 2 B^2 } \\
& \geq \frac{\sqrt{0.5}\left(\phi_2^2 + \phi_3^2\right)^{1/2}}{B^2 +   \gamma 2 B^2 },
\end{align*}where $B=1$ is the bound on the feature vectors. The above quantity is lower bounded by a positive constant (since the numerator is $>0$ and the denominator is bounded by $3$). Meanwhile, the value of the $L_2$ norm $ \norm{\phi_1\begin{pmatrix}
\sqrt{\mu_1} P_{1,1}\\
    0\\
     0
\end{pmatrix}  + \phi_2 \begin{pmatrix}
 0\\
\sqrt{\mu_2} P_{2,2}\\
  0
\end{pmatrix} + \phi_3\begin{pmatrix}
 \frac{\mu_3}{\sqrt{\mu_1}}P_{3,1}\\
 \frac{\mu_3}{\sqrt{\mu_2}}P_{3,2}\\
 \sqrt{\mu_3} P_{3,3}
\end{pmatrix}}_2$ goes to $\infty$ via the term $1/\sqrt{\mu_1}$, so $\rho(0) \rightarrow \infty$. For the other asymptote, taking $\mu_1 \rightarrow 1$ gives us that the ratio on the left goes to 
\[
\frac{\phi_1}{|\phi_1^2 - \gamma \phi_1 \phi'_1|} = \frac{1}{|\phi_1 - \gamma \phi'_1|}.
\]
and the $L_2$ norm goes to $\phi_1P_{1,1}$, so the product goes to 
\[
\frac{\phi_1 P_{1,1}}{|\phi_1 - \gamma \phi'_1|} \leq \frac{0.6}{|1 - 0.9 \frac{\phi'_1}{\phi_1}|}.
\]
Using $P_{1,1} \approx 0.58... \leq 0.6$ and $\gamma = 9/10$. It remains to deal with the term $|1 - 0.9 \phi_1' / \phi_1|$ and to show that this term can be made arbitrarily large (so that the ratio can be bounded by an arbitrarily small number). Towards this, note that
\begin{align*}
\phi_1' / \phi_1 &= \frac{P_{1,1} \phi_1 + P_{1,4} \phi_4 + P_{1,5}\phi_5}{\phi_1} = P_{1,1} + P_{1,4} \frac{\phi_4}{\phi_1} + P_{1,5} \frac{\phi_5}{\phi_1} \\
&= P_{1,1} + P_{1,4} \frac{\lambda_1 d_4^4 + \lambda_2 d_5^4 + \lambda_3 \psi^4}{\lambda_1 d_4^1 + \lambda_2 d_5^1 + \lambda_3 \psi^1} + P_{1,5} \frac{\lambda_1 d_4^5 + \lambda_2 d_5^5 + \lambda_3 \psi^5}{\lambda_1 d_4^1 + \lambda_2 d_5^1 + \lambda_3 \psi^1}
\end{align*}
Note that the ratios $\frac{\lambda_1 d_4^4 + \lambda_2 d_5^4 + \lambda_3 \psi^4}{\lambda_1 d_4^1 + \lambda_2 d_5^1 + \lambda_3 \psi^1}$ and $\frac{\lambda_1 d_4^5 + \lambda_2 d_5^5 + \lambda_3 \psi^5}{\lambda_1 d_4^1 + \lambda_2 d_5^1 + \lambda_3 \psi^1}$ are continuous functions of $c$, since $c$ only appears via $\lambda_3 = c$ or $\lambda_2 = \left( c \sign\langle a, M_3(\psi) \rangle - b_1 \right) / b_2$. Thus, taking $c$ small enough towards $0$ we can get arbitrarily close to the value 
\[
    \phi_1' / \phi_1 \rightarrow  P_{1,1} + P_{1,4} \frac{d_4^4 - \frac{b_1}{b_2} d_5^4}{d_4^1 -\frac{b_1}{b_2} d_5^1} + P_{1,5} \frac{d_4^5 -\frac{b_1}{b_2} d_5^5}{d_4^1 -\frac{b_1}{b_2} d_5^1}
\]
Recall that $b$ is the vector obtained by the rank-1 factorization of $\bar{M} = a b^\top$, i.e. is defined by $b_1 = \langle p_4, d_4 \rangle_\mu$ and $b_2 = \langle p_4, d_5 \rangle_\mu$. Solving for $b_1$ and $b_2$ in the case where $\mu_1 = 1$ gives us $b_1 = p_4^1 d_4^1$ and $b_2 = p_4^1 d_5^1$. Plugging this in, we have
\begin{align*}
\phi_1' / \phi_1 &\rightarrow  P_{1,1} + P_{1,4} \frac{d_4^4 - \frac{d_4^1}{d_5^1} d_5^4}{d_4^1 -\frac{d_4^1}{d_5^1} d_5^1} + P_{1,5} \frac{d_4^5 -\frac{d_4^1}{d_5^1} d_5^5}{d_4^1 -\frac{d_4^1}{d_5^1} d_5^1} \\
&=  P_{1,1} + P_{1,4} \frac{d_4^4 - \frac{d_4^1}{d_5^1} d_5^4}{d_4^1 - d_4^1} + P_{1,5} \frac{d_4^5 -\frac{d_4^1}{d_5^1} d_5^5}{d_4^1 -d_4^1} = \infty,
\end{align*}

as desired.

Concluding, we have that by the intermediate value theorem, for any value of $x$ between $(0, \infty)$, there exists some value of $\mu_1$ and $c$ s.t. the expression above can take value $\rho(\mu_1) = x$, so in particular we can satisfy $x = \norm{\Pi_\mu P}_\mu / \sigmamin{\Sigma^{-1/2}A\Sigma^{-1/2}}$ with some choice of $\mu$.

\section{Cases where $\alpha=1$ is asymptotically achievable}\label{app:alphaequals1}

Thanks to the proof of Equation \ref{eq:lstd-ub}, we identify several scenarios where the true solution can be recovered. 

\begin{enumerate}
    \item $\Phi A^{-1} \Phi^\top D P v^\perp = 0$, in particular $\Phi^\top D P v^\perp = 0$, e.g. when is satisfied under the condition that the orthogonal subspace of $\col(\Phi)$ is closed under $P$ (i.e. $P$ maps orthogonal vectors (of the features) to orthogonal vectors). Then LSTD has an $L_2(\mu)$ approximation factor of $1$.
    	\begin{itemize}
    		\item Proof: from Equation \eqref{eq:ls-minus-lstd}.
    	\end{itemize}
    \item $\norm{P}_\mu < \infty$ implies that $v_\cM$ can be learned exactly on the support of $\mu$. Thus with the tabular function class the asymptotic approximation ratio in the $L_2(\mu)$ norm is either $1$ if $\norm{P}_\mu < \infty$ or $\infty$ if $\norm{P}_\mu$.
		\begin{itemize}
    		\item Proof: If $\norm{P}_\mu < \infty$ then we must have the condition $(\mu(s)>0 \,\&\, P(s'|s)>0) \implies \mu(s')>0$, otherwise in the equation $\max_{\norm{v}_\mu=1} \norm{Pv}_\mu$ we will have a contribution of $P_{s,s'}v(s')$ for some unsupported state $s'$, and the value for $v(s')$ can be taken to infinity while satisfying the constraint $\norm{v}_\mu=1$. From this condition it is easy to see that $v_\cM$ can be recovered exactly on $\mu$, as in the asymptotic regime we have access to $r(s) \forall s \in \mu$ and $P(s) \forall s \in \mu$, and if a state transitions to $s'$ then we will also have $P(s')$ and $r(s')$.%
    	\end{itemize}

\end{enumerate}

\section{Proofs for Section \ref{sec:linfty}}
\subsection{Proof of Theorem \ref{thm:lstd-ub-infty}}\label{app:lstd-ub-infty}

\lstdubinfty*

\begin{proof}
We repeat the steps of Lemma \ref{lemma:lstd-ls}. Let us write $v_\cM = \Pi_\infty v_\cM + \delta \coloneqq \Phi \theta_\infty + \delta$, so that $\delta = v_\cM - \Pi_\infty v_\cM$.

Then we have:
\begin{align*}
    (I - \gamma P) v_M &= r \\
    (I - \gamma P) \Phi \theta_\infty + (I- \gamma P) \delta &= r \\
    \Phi^\top D (I - \gamma P) \Phi \theta_\infty + \Phi^\top D (I - \gamma P) \delta &= \Phi^\top D r \tag{$\Phi^\top D$ on both sides}\\
    \left(\Phi^\top D(\Phi - \gamma P \Phi)\right) \theta_\infty &= \Phi^\top D r - \Phi^\top D(I-\gamma P) \delta\\
     A \theta_\infty  &= b - \Phi^\top D(I-\gamma P) \delta \tag{Defns of $A,b$}\\
     \theta_\infty &= A^{-1} b - A^{-1}\Phi^\top D(I-\gamma P) \delta \tag{$A^{-1}$ exists}
\end{align*}
Meanwhile, the LSTD solution is defined by $\LSTD = A^{-1}b$. Substracting both of these gives:
 \begin{equation}\label{eq:infty-lstd}
 \theta_\infty - \LSTD = -A^{-1}\Phi^\top D(I-\gamma P) \delta
 \end{equation}
Now, applying $\Phi$ and taking the $\infty$ norm gives
\begin{align*}
\norm{\Phi(\theta_\infty - \LSTD)}_\infty &= \norm{\Phi A^{-1}\Phi^\top D(I-\gamma P) \delta}_\infty \\
&\leq \norm{\Phi A^{-1}\Phi^\top D(I-\gamma P)}_\infty \norm{\delta}_\infty \\
&\leq \norm{\Phi A^{-1}\Phi^\top D}_\infty\norm{(I-\gamma P)}_\infty \norm{\delta}_\infty \\
&\leq \norm{\Phi A^{-1}\Phi^\top D}_\infty (1+\gamma) \norm{\delta}_\infty \\
\end{align*}

It remains to relate $\norm{\Phi A^{-1}\Phi^\top D}_\infty$ to $\sigmamin{A}$. Notice that
$$
(\Phi A^{-1}\Phi^\top D)_{i,j} = \mu_j \left\langle \phi_i, A^{-1}\phi_j\right\rangle,
$$
The $L_\infty$ matrix norm is the maximum $L_1$ norm of a row, thus
\begin{align*}
\norm{\Phi A^{-1}\Phi^\top D}_\infty = \max_i \left( \sum_j | \mu_j \left\langle \phi_i, A^{-1}\phi_j\right\rangle | \right) &\leq \max_i\left( \sum_j \mu_j \norm{\phi_i}_2 \norm{A^{-1}\phi_j}_2 \right) \tag{Cauchy-Schwartz} \\
&= \norm{A^{-1}}_2 \max_i \norm{\phi_i}_2 \left(\sum_j \mu_j \norm{\phi_j}_2\right) \\
&\leq \norm{A^{-1}}_2 1 \tag{$\norm{\phi_i}\leq 1 \, \forall i$} \\
&= \frac{1}{\sigmamin{A}} 
\end{align*}

Combining everything and using a triangle inequality gives us:
$$
\norm{\vLSTD - v_\cM}_\infty \leq (1 + \frac{(1+\gamma)1}{\sigmamin{A}}) \norm{\Phi \theta_\infty - v_\cM }_\infty
$$
\end{proof}

\if0
\subsection{Extension to $\sigmamin{A} > 0$}

\begin{theorem}
$\forall y \in (0,\infty)$ there exists a family of instances $\bM = \{(\cM,\Phi,\mu)\}$ which all satisfy $\sigmamin{A}=y$, yet any estimator $\hat{v}$ will satisfy have 
	$$
\sup_{(\cM,\mu,\phi)\in \bM} \alpha_\infty(\hat v; (\cM,\mu,\phi)) \geq \min\{1,\frac{1}{y}\}
	$$
\end{theorem}
\begin{proof}
In the two-state example, take the first state to have feature $\phi_1 = 1$ and the second state to have feature $\phi_2 = 1/\gamma+\varepsilon$. This gives $A = 1 - \gamma ( 1 \cdot (1/\gamma+\varepsilon)) = -\gamma \varepsilon$, which has $y \coloneqq \sigmamin{A} = \gamma \varepsilon$. The misspecification is now
\begin{align*}
\inf_\theta \norm{\Phi \theta - v_\cM}_\infty  &= \inf_\theta \max\{ \left| \theta - \gamma \frac{r}{1-\gamma} \right|, \left|\theta(1/\gamma+\varepsilon) - \frac{r}{1-\gamma}\right| \}  \\
&= \inf_\theta \max\{ \gamma \left| \theta/\gamma - \frac{r}{1-\gamma} \right|, \left|\theta/\gamma - \frac{r}{1-\gamma} + \varepsilon \theta \right| \} \\
&\leq \varepsilon \left|\gamma \frac{r}{1-\gamma} \right| \tag{Picking $\theta = \gamma r/(1-\gamma)$}\\
&\leq \frac{\varepsilon \gamma }{1-\gamma}
\end{align*}

The problem is still unidentifiable, which means any estimator will have an error of at least $1/(1-\gamma)$. Thus, the worst-case approximation ratio is 
$$
\frac{\norm{\hat{v} - v_\cM}_\infty}{\inf_\theta \norm{\Phi \theta - v_\cM}_\infty} \geq \frac{\frac{1}{1-\gamma}}{\frac{ \gamma \varepsilon}{1-\gamma}} = \frac{1}{\gamma \varepsilon} = \frac{1}{y}
$$

\end{proof}
\fi

\subsection{Proof of Theorem \ref{thm:sigmamin-infty}}\label{app:sigmamin-infty}

\sigmamininfty*

    \begin{proof}
We take $P = \begin{pmatrix} 
0 & 1 \\
0 & 1 \\
\end{pmatrix}$ and $D = \begin{pmatrix} 
1 & 0 \\
0 & 0 \\
\end{pmatrix}$. Note that this is the same MRP as in \cite{amortila2020variant} and Lemma \ref{lemma:PiPnecessary}. This gives a discounted occupancy matrix
\[
\bd = (I - \gamma P)^{-1} = \begin{pmatrix}
1 & \gamma/(1-\gamma) \\
0 & 1/(1-\gamma) 
\end{pmatrix}.
\]
Let $d_1$ denote the first column of $\bd$ and $d_2$ denote the second column. We take $r = (0, r_2)^\top$, i.e. no reward at state $1$ and a reward of $r_2$ at state $2$. This gives $v_\cM = r_2 d_2$. We set one instance to have $r_2 = 1$, one instance to have $r_2 = 0$, and the last instance to have $r_2 = -1$. The three instances are otherwise identical. We take $\Phi = \left[\alpha d_1 + d_2\right](1-\gamma) \in \bR^{2 \times 1}$ (thus $\phi(s) \in \bR$), and we will later impose that $0 \leq  \alpha \leq 1$. Assuming that this bound on $\alpha$ holds for now, we can see that $\norm{\phi_1}_2 = [\alpha + \gamma/(1-\gamma)](1-\gamma) \leq (1-\gamma)/(1-\gamma) = 1$ and $\norm{\phi_2}_2 = (1-\gamma)/(1-\gamma)$ and thus $\norm{\phi_i}_2 \leq 1$ for all $i$. We can directly verify that 
\[
A = \Phi^\top D (I - \gamma P)\Phi = \left[\alpha^2 + \alpha \gamma/(1-\gamma)\right](1-\gamma)^2 = \alpha^2(1-\gamma)^2 + \alpha \gamma (1-\gamma) %
\]
We need $\sigmamin{A} = |A| = A = y$, so we can solve the quadratic for $\alpha$ and pick the positive solution to get:
\[
\alpha = \frac{-\gamma + \sqrt{\gamma^2+4y}}{2(1-\gamma)}
\]

Note that $\alpha$ satisfies the bound $0 \leq \alpha \leq 1$ whenever $\gamma < 1$ and $0 \leq y \leq 1- \gamma$, which holds by the assumption in our theorem statement.
The misspecification error is at most:
$$
\inf_\theta \norm{v_\cM - \Phi \theta}_\infty = \inf_\theta \norm{r_2 d_2 - (\alpha d_1 + d_2)\theta(1-\gamma) }_\infty = \inf_\theta \norm{(r_2 - (1-\gamma)\theta)d_2 - \alpha(1-\gamma)\theta d_1}_\infty \leq \alpha \norm{d_1}_\infty = \alpha,
$$
where the upper bound was obtained by plugging in $\theta = \frac{r_2}{1-\gamma}$. 
Note that the minimax estimator against these three instances will need to output $\theta = 0$ since the instance with $r_2 = 0$ is realizable with $\theta = 0$. Namely, if the learner does not output $\theta=0$ then its worst-case approximation will be $\infty$. This gives the ratio:
\begin{align*}
\alpha_\infty &\geq \frac{\norm{v_\cM - 0}_\infty}{\norm{\Pi_\infty v_\cM - v_\cM}_\infty} \\
&\geq \frac{\norm{v_\cM}_\infty}{\alpha } \\
&= \frac{1}{\alpha(1-\gamma)} \\
&= \frac{1}{y} \left\{\alpha(1-\gamma) + \gamma \right\} \tag{using that $\alpha(1-\gamma)\left\{\alpha(1-\gamma) + \gamma\right\} = y$ by definition of $A$}\\
&= \frac{1}{y} \left\{\frac{-\gamma + \sqrt{\gamma^2+4y}}{2} + \gamma \right\}  \tag{using that $\alpha =  \frac{-\gamma + \sqrt{\gamma^2+4y}}{2(1-\gamma)} $}\\
&= \frac{\gamma}{2y} \left\{1 + \sqrt{1+\frac{4y}{\gamma^2}} \right\}\\
&\geq \frac{\gamma}{2y} \left\{1 + 1+\frac{2y}{\gamma^2}- \frac{(4y)^2}{8\gamma^4} \right\} \tag{using that $\sqrt{1+x} \geq 1 + x/2 - x^2/8$ for all $x \geq 0$}\\
&= \frac{\gamma}{y} \left\{1+\frac{y}{\gamma^2}- \frac{y^2}{\gamma^4} \right\} \\
&= \frac{\gamma}{y} +\frac{1}{\gamma}- \frac{y}{\gamma^3} \\
&\geq \frac{\gamma}{y} +\frac{1}{\gamma}- \frac{1-\gamma}{\gamma^3} \tag{using that $y \leq 1-\gamma$} \\
&\geq \frac{\gamma}{y} + \frac{1}{2}, \tag{using that $\frac{1}{\gamma}- \frac{1-\gamma}{\gamma^3} \geq \frac{1}{2}$ when $\gamma \geq c_1$}
\end{align*}
as desired. The value of $c_1$ can be taken to be the smallest solution $x$ such that $\frac{1}{x} - \frac{1-x}{x^3} \geq \frac 1 2$, which by Mathematica is approximately $0.6889$ (but one can verify that $0.7$ suffices and that this inequality holds for all $x \geq 0.7$ since the function is increasing).
\end{proof}

\subsection{Proof of Theorem \ref{thm:aliasing-ub-infty}}\label{app:aliasing-ub-infty}

\aliasingubinfty*

\begin{proof}
Inspired by the theory of ``$q^\star$-irrelevant abstractions'' \cite{li2006towards,jiang2018notes,xie2021batch}, we define $v_\cM$-irrelevant abstractions as follows:
\begin{definition}\label{def:irr-abs}
A feature map $\varphi: \cS \mapsto \cX$ is an $\varepsilon-$approximate $v_\cM-$irrelevant abstraction for MRP $\cM$ if there exists a function $f: \cX \mapsto \bR$ such that 
$$ 
\inf_{f: \cX \mapsto \bR}\norm{f \circ \varphi - v_\cM}_\infty = \varepsilon 
$$
\end{definition}

Note that the $\inf$ is taken over all pointwise functions over $\cX$, and thus every feature mapping with an $L_\infty$-misspecification error of $\varepsilon$ is also a $\varepsilon$-approximate $v_\cM$-irrelevant abstraction. 

To conclude the proof we use Theorem 5 from \cite{jiang2018notes}, which establishes the analogous claim for the case of $q^\star$-irrelevant abstractions. Indeed, the case of $v_\cM$-irrelevant abstractions can be reduced from the more general case of $q^\star$-irrelevant abstractions by considering the case where there is only one action to take in each state. It is easily seen that our Bayes model is then equivalent to the model constructed in Lemma 3 of \cite{jiang2018notes}, which Theorem 5 uses to establish the approximation error bound of $2/(1-\gamma)$. 
\end{proof}

\subsection{Proof of Theorem \ref{thm:aliasing-lb-infty}}\label{app:aliasting-lb-infty}

\aliasinglbinfty*

\begin{proof}
The first MRP $\cM_1$ is defined as
\begin{center}
\begin{tikzpicture}
\node[state] (q1) {$\varphi(s_1) = \varphi$};
\node[state, right=1in of q1] (q2) {$\varphi(s_2)=\varphi$};
\draw[->] (q1) edge[above] node{1} (q2)
(q2) edge[loop above] node{0} (q2);
\end{tikzpicture}\label{fig:aliasing-lb-infty}
\end{center}
We set $\phi = 1$ for simplicity. We place the initial distribution $\mu(s_1) = p$ and $\mu(s_2) = 1-p$, and should think of $p \rightarrow 1$ (we can't actually set $p = 1$ due to the full-support condition, but a limiting argument suffices). Note that $v_\cM(s_1) = 1$ and $v_\cM(s_2) = 0$, so the optimal $\infty$-norm approximation for this MRP is $\theta_1 = \phi \theta = \frac{1}{2}$. 

Our second MRP $\cM_2$ is the following:
\begin{center}
\begin{tikzpicture}
\node[state] (q1) {$\varphi$};
\draw[->] (q1) edge[loop above] node{$\text{Ber}(p)$} (q1);
\end{tikzpicture}
\end{center}

which generates the same distribution $\bQ$. This instance is realizable with value function $\theta_2 = v_\cM = \frac{p}{1-\gamma}$, which forces our estimator to output $\theta_2$. Let $p$ be large enough such that $\norm{\theta_2-v_{\cM_1}}_\infty = \max\{|\frac{p}{1-\gamma} - 1|, |\frac{p}{1-\gamma} - 0|\} = \frac{p}{1-\gamma}$ (i.e. $p > (1-\gamma)/2$). Taking the ratio of approximation errors:
$$
\frac{\norm{\theta_2 - v_{\cM_1}}_\infty}{\norm{\theta_1 - v_{\cM_1}}_\infty}= \frac{p/(1-\gamma)}{1/2} = \frac{2p}{1-\gamma} \geq \frac{2}{1-\gamma} - \varepsilon,
$$
where the last step takes $p \geq 1 - \frac{\varepsilon(1-\gamma)}{2}$. 
\end{proof}

\subsection{Proof of Corollary \ref{cor:projected-v-phi}}\label{app:projected-v-phi}

\begin{restatable}{corollary}{projectedvphi}\label{cor:projected-v-phi}
The projected Bayes value function has an approximation factor of
$$
\norm{ \Pi_\infty (v_\varphi \circ \phi) - v_\cM}_\infty \leq \left(1+\frac{2}{1-\gamma}\right) \inf_\theta \norm{\Phi \theta - v_\cM}_\infty
$$
\end{restatable}
\begin{proof}
This amounts to an application of the triangle inequality:
\begin{align*}
\norm{ \Pi_\infty (v_\varphi \circ \phi) - v_\cM}_\infty &\leq \norm{ \Pi_\infty v_\varphi - \Pi_\infty v_\cM}_\infty + \norm{ \Pi_\infty v_\cM - v_\cM}_\infty \\
&\leq \norm{ (v_\varphi \circ \phi) - v_\cM}_\infty + \inf_\theta \norm{\Phi \theta - v_\cM}_\infty \tag{$\Pi_\infty$ is non-expansive}\\
&\leq \frac{2}{1-\gamma} \inf_\theta \norm{\Phi \theta - v_\cM}_\infty + \inf_\theta \norm{\Phi \theta - v_\cM}_\infty \tag{Previous bound}\\
&= \left(1+\frac{2}{1-\gamma}\right) \varepsilon_\infty,
\end{align*}
which concludes the proof.
\end{proof}

\section{Translating $L_2(\mu)$ oracle inequalities to $L_\infty$ oracle inequalities}\label{app:translating}

This section shows that one can convert an $L_2(\mu)$ oracle inequality to an $L_\infty$ oracle inequality

\begin{lemma}
Assuming we have a bound
$$
\norm{v_\theta - v_\cM}_\mu \leq \alpha_\mu \norm{\Pi_\mu v_\cM - v_\cM}_\mu.
$$
This can be converted to an approximation ratio bound 
$$
\norm{v_\theta - v_\cM}_\infty \leq \left(1 + \max_s \norm{\Sigma^{-1/2}\phi(s)}_2(1+\alpha_\mu) \right)\norm{\Pi_\infty v_\cM - v_\cM}_\infty
$$ 
\end{lemma}
\begin{proof}
Let $\Phi\theta_\cM$ be an $L_\infty$ linear projection, and $\delta(s)$ be such that $v_\cM(s) = \delta(s) + \theta_\cM^\top \phi(s)$.
\begin{align*}
\norm{v_\cM - v_\theta}_\infty &= \max_s |\theta^\top \phi(s) - v_\cM(s)|\\
&= \max_s | \theta^\top \phi(s) - \theta_\cM^\top \phi(s) - \delta(s) | \\
&\leq \norm{\delta(s)}_\infty + \max_s | (\theta^\top \phi(s) - \theta_\cM)^\top \phi(s)| \\
&\leq \norm{\delta(s)}_\infty + \max_s \norm{\Sigma^{-1/2}\phi(s)}_2 \norm{\Sigma^{1/2}(\theta_\cM - \theta)}_2 \tag{Cauchy-Schwartz} \\
&= \norm{\delta(s)}_\infty + \left(\max_s \norm{\Sigma^{-1/2}\phi(s)}_2\right) \norm{\Phi(\theta_\cM - \theta)}_\mu \\
&\leq \norm{\delta(s)}_\infty + \left(\max_s \norm{\Sigma^{-1/2}\phi(s)}_2\right) \left(\norm{\Phi(\theta_\cM) - v_\cM}_\mu + \norm{v_\cM - \Phi\theta}_\mu\right)\\
&\leq \norm{\delta(s)}_\infty + \left(\max_s \norm{\Sigma^{-1/2}\phi(s)}_2\right) \left(\norm{\Phi(\theta_\cM) - v_\cM}_\mu + \alpha_\mu\norm{v_\cM - \Pi_\mu v_\cM}_\mu\right)\\
&\leq \norm{\delta(s)}_\infty + \left(\max_s \norm{\Sigma^{-1/2}\phi(s)}_2\right) \left(\norm{\Phi(\theta_\cM) - v_\cM}_\mu + \alpha_\mu\norm{v_\cM - \Phi\theta_\cM}_\mu\right) \tag{$\Pi_\mu v_\cM = \inf_{\hat{v} \in \cF_\Phi}\norm{v_\cM - \hat{v}}$}\\
&\leq \norm{\delta(s)}_\infty + \left(\max_s \norm{\Sigma^{-1/2}\phi(s)}_2\right) \left((1+\alpha_\mu)\norm{\Phi(\theta_\cM) - v_\cM}_\mu \right) \\
&\leq \norm{\delta(s)}_\infty + \left(\max_s \norm{\Sigma^{-1/2}\phi(s)}_2\right) (1+\alpha_\mu)\norm{\Phi(\theta_\cM) - v_\cM}_\infty \\
&= \norm{\delta(s)}_\infty + \left(\max_s \norm{\Sigma^{-1/2}\phi(s)}_2\right) (1+\alpha_\mu)\norm{\delta(s)}_\infty \\
&= \left(1 + \max_s \norm{\Sigma^{-1/2}\phi(s)}_2(1+\alpha_\mu) \right)\norm{\delta(s)}_\infty
\end{align*}
\end{proof}

\end{document}